%% file: main.tex
\RequirePackage{luatex85,shellesc} 

\documentclass[pdflatex,sn-basic]{sn-jnl}

\catcode`\|=12\relax 
\usepackage{common}
\usepackage{mymacro}
\usepackage[hide]{notes}

\jyear{2022}%

\theoremstyle{thmstyleone}%
\newtheorem{theorem}{Theorem}
\newtheorem{lemma}[theorem]{Lemma}%
\newtheorem{corollary}[theorem]{Corollary}%

\theoremstyle{thmstyletwo}%
\newtheorem{remark}{Remark}%
\newtheorem{problem}{Problem}%

\theoremstyle{thmstylethree}%

\algrenewcommand\algorithmicrequire{\textbf{Input:}} 
\algrenewcommand\algorithmicensure{\textbf{Output:}}

\begin{document}

\title[Accurate decision trees with complexity guarantees]%
{\revise{Regularized} impurity reduction: Accurate decision trees with complexity guarantees}


\author*[1]{\fnm{Guangyi} \sur{Zhang}}\email{guaz@kth.se}

\author[1]{\fnm{Aristides} \sur{Gionis}}\email{argioni@kth.se}


\affil[1]{\orgdiv{Computer Science}, \orgname{KTH Royal Institute of Technology}, \orgaddress{\city{Stockholm}, \country{Sweden}}}




\abstract{
Decision trees are popular classification models,
providing high accuracy and intuitive explanations.
However, as the tree size grows the model interpretability deteriorates.
Traditional tree-induction algorithms, such as C4.5 and CART, 
rely on impurity-reduction functions that promote the discriminative power of each split.
Thus, although these traditional methods are accurate in practice, there has been no theoretical guarantee that they will produce small trees.
In this paper, we justify the use of a general family of impurity functions, 
including the popular functions of entropy and Gini-index,
in scenarios where small trees are desirable,
by showing that a simple enhancement can equip them with complexity guarantees.
We consider a general setting, 
where objects to be classified are drawn from an arbitrary probability distribution, 
classification can be binary or multi-class, 
and splitting tests are associated with non-uniform costs.
As a measure of tree complexity, we adopt the expected cost to classify an object drawn from the input distribution, 
which, in the uniform-cost case, is the expected number of tests. 
We propose a tree-induction algorithm that gives a logarithmic approximation guarantee on the tree complexity.
This approximation factor is tight up to a constant factor under mild assumptions.
The algorithm recursively selects a test that maximizes a greedy criterion 
defined as a weighted sum of three components.
The first two components encourage the selection of tests 
that improve the balance and the cost-efficiency of the tree, respectively, 
while the third impurity-reduction component encourages the selection of more discriminative tests.
As shown in our empirical evaluation, compared to the original heuristics,
\revise{the enhanced algorithms strike an excellent balance between predictive accuracy and tree complexity.}
}




\maketitle

\input{intro}
\input{related}
\input{definition}
\input{algorithm}
\section{Approximation guarantee}
\label{section:analysis}
\input{analysis}
\input{experiment}
\input{conclusion}

\backmatter

%
%
%

\bmhead{Acknowledgments}
%
This research is supported by the Academy of Finland projects AIDA (317085) and MLDB (325117),
the ERC Advanced Grant REBOUND (834862), 
the EC H2020 RIA project SoBigData++ (871042), 
and the Wallenberg AI, Autonomous Systems and Software Program (WASP) 
funded by the Knut and Alice Wallenberg Foundation.

\section*{Declarations}


\begin{itemize}
\item Funding: 
This research is supported by the Academy of Finland projects AIDA (317085) and MLDB (325117),
the ERC Advanced Grant REBOUND (834862),
the EC H2020 RIA project ``SoBigData++'' (871042), and the
Wallenberg AI, Autonomous Systems and Software Program (WASP).
The funders had no role in study design, data collection and analysis, 
decision to publish, or preparation of the manuscript. 
\item Competing interests: 
The authors have no conflicts of interest to declare that are relevant to the content of this article.
\item Ethics approval: Not applicable
\item Consent to participate: Not applicable
\item Consent for publication: Not applicable
\item Availability of data and materials: 
All datasets we use are publicly available in the UCI Machine Learning Repository~\citep{UCI} and OpenML~\citep{OpenML2013}.
\item Code availability: 
Our implementation is publicly available in the Github repository.%
\footnote{\url{https://github.com/Guangyi-Zhang/low-expected-cost-decision-trees}}
\item Authors'~contributions: 
Guangyi Zhang is responsible for the theoretical and experimental development.
Both authors contribute significantly to the design and writing of the work.
\end{itemize}

%
%
%
%
%

\begin{appendices}

\input{supp.tex}

%




\end{appendices}


\bibliography{references}


\end{document}

%% file: intro.tex
\section{Introduction}

Decision trees are 
known to provide a good trade off between accuracy and interpretability. 
However, when their size grows, decision trees become harder to interpret,
preventing their deployment in safety-critical applications and in domains
where model transparency is highly valued, such as disease diagnosis.
As interpretability still remains an ill-defined notion~\citep{lipton2018mythos},  
in this paper we consider tree complexity, a commonly-accepted proxy, to quantify interpretability \citep{freitas2014comprehensible,doshi2017towards}.
In addition, low tree-complexity promotes cheaper and faster evaluation.
Note that post-pruning techniques, such as the standard minimal cost-complexity pruning \citep{breiman1984classification}, are heuristics performed mainly to avoid overfitting.
Therefore, in order to produce interpretable trees, 
we aim for an \revise{integrated} tree-induction algorithm that considers both the accuracy and complexity of the inferred~trees.

\begin{figure}[t] 
\centering
\subcaptionbox{Standard C4.5 (AUC ROC 0.779 and expected height 5.9)}{
   	\picinput[width = 0.9\textwidth]{viz_speed-dating_C45}
	\label{fig:demo:a}}
\newline
\subcaptionbox{Enhanced C4.5 (in this paper) (AUC ROC 0.790 and expected height 4.9)}{
   	\picinput[width = 0.9\textwidth]{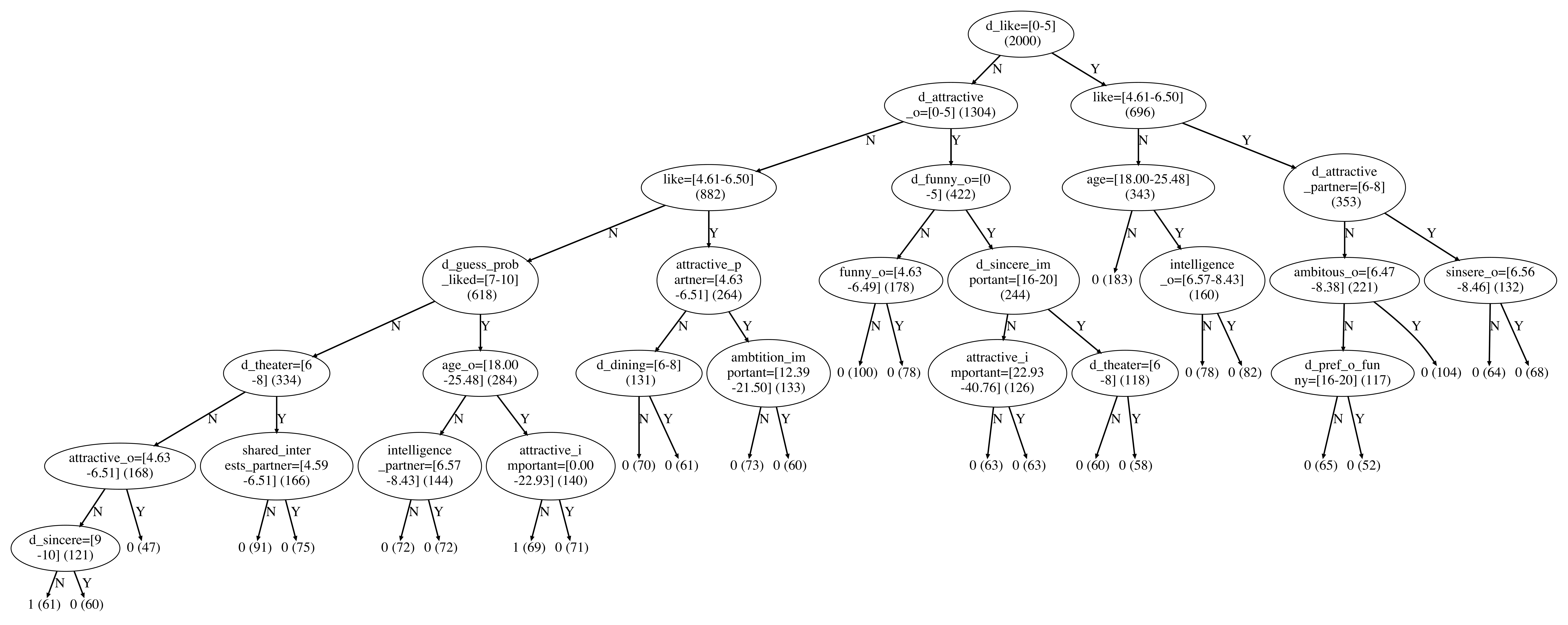}
	\label{fig:demo:b}}
\caption{Decision trees for predicting if participants would like to see their date again after speed dating.
Each internal node includes the test used and the number of participants in parentheses.
Leaf nodes make a decision.}
\label{fig:demo}
\end{figure}

More concretely, given a set of labeled objects (examples)
drawn from an arbitrary probability distribution, 
our goal is to learn a decision tree that outputs the correct class of a given input object.
Each internal tree node is equipped with a single test,
e.g., a projection split along a feature, and each test is associated with a non-uniform cost, 
i.e., the cost of evaluating the outcome of the test.
For example, an input object may represent a person, a test may correspond to a blood sugar test, and one possible outcome can be ``high.''
Our aim is to learn trees that are accurate and have low complexity.
The latter complexity objective is measured by the expected cost to classify an object drawn from the input distribution;
if all tests incur the same cost, 
this measure is simply the expected number of tests to classify an object.
This complexity measure reflects \revise{a form of ``local''} interpretability:
the more tests are involved in an if-then rule for a given object, the more obscure the rule becomes to a user~\citep{freitas2014comprehensible}.
Figure \ref{fig:demo} helps demonstrating this intuition by juxtaposing two decision trees with different complexity.
Note that non-uniform test costs may arise in different real-world scenarios; 
for example, in a medical-diagnosis application
some tests can be significantly more expensive than~others.

The problem of minimizing the expected cost of a tree for perfect class identification has been extensively studied.
Typically, the assumption of \emph{realizability} (or consistency) is being made, 
which states that for every two distinct objects there exists at least one test that can distinguish them.
Thus, one can always expand the tree until it classifies every object in the training data perfectly.
Then, the goal is to find the tree with the minimum expected cost 
that classifies each object perfectly.
Note that, in practice, the realizability assumption can be easily fulfilled by data preprocessing, 
as we demonstrate later in our experiments.
When each object belongs to a distinct class, the problem is referred to as \emph{entity identification} (\entid) \citep{gupta2017approximation}.
Without this restriction, the problem is called \emph{group identification} (\clsid) \citep{cicalese2014diagnosis}.
A further generalization that is called \emph{adaptive submodular ranking} (\asr) \citep{navidi2020adaptive}
characterizes the tree-building process as interaction among multiple submodular functions, one for each object, 
and achieves logarithmic approximation by a greedy algorithm.
The above-mentioned works \citep{gupta2017approximation,cicalese2014diagnosis,navidi2020adaptive} 
consider the problem of building trees for the purposes of exact identification. 
They do not consider issues of accuracy and overfitting. 
In fact, exact identification on a set of (training) data leads precisely to overfitting.

The algorithm proposed for the \asr problem by \citet{navidi2020adaptive}
provides a very elegant solution for the identification task it has been designed for.
However, in practice, it is not suitable for classification tasks in the context of statistical learning, 
because the chosen tests are geared towards small expected cost and are not necessarily discriminative.
Discriminative power is generally measured by the homogeneity of the target variable within a tree node, and is essential for the generalization of model performance over unseen data.
Their method selects splits that minimize the number of \emph{heterogeneous pairs} (also known as impure pairs) of objects~\citep{golovin2010near,cicalese2014diagnosis}.
In Section \ref{section:supplementary:asr-splits} 
we provide a simple example where the criterion favors a non-discriminative (presumably random) test 
over a discriminative one.
While random tests lead to a balanced tree with bounded expected depth,  
they are not ``learning'', that is, no statistical dependence is captured between tests and the target variable.

On the other hand,
traditional decision-tree methods,  
such as \cart~\citep{breiman1984classification} and \cfourfive~\citep{quinlan1993c4.5}, 
rely on time-tested impurity-reduction heuristics
that yield decision trees with high discriminative power.
Although trees produced by these popular methods are accurate in practice, 
there has been no guarantee on the size, or depth, of the resulting trees.
Actually, despite the popularity of these methods, their theoretical properties remain still 
poorly understood~\citep{bellala2012group,brutzkus2019on,blanc2020top}.

In this paper we propose a general family of methods that achieve the best of both worlds: 
\emph{it produces decision trees having both high accuracy and bounded depth}. 
Our key discovery is that the \asr framework can be extended 
to effectively analyze a broad range of impurity functions for tree induction.

More formally, we introduce the \emph{non-overfitting group identification} (\nci) problem,
which is a natural generalization of group identification (\clsid),
where we further allow early termination during tree expansion to avoid overfitting. 
We propose a novel greedy algorithm that takes into consideration the impurity reduction
and maintains the strong approximation guarantee on the complexity of the resulting tree.
Specifically, our greedy algorithm admits the use of a \emph{general family of decomposable impurity functions}, 
which is defined to be in the form of a weighted sum over impurity scores in each class.
This family includes the popular functions of entropy and Gini-index.
Therefore, our approach generalizes many traditional tree-induction algorithms
such as \cart and \cfourfive into a complexity-aware method.

\smallskip
In concrete, in this paper we make the following contributions.
\begin{itemize}
	\item 
	We extend the adaptive submodular ranking (\asr) framework of \citet{navidi2020adaptive}
	and we propose a novel greedy algorithm to select discriminative tests for 
	the {non-overfitting group identification} (\nci) problem.
	Our algorithm offers an asymptotically tight approximation guarantee on the complexity of the inferred~tree under mild assumptions.
	\item We define a general family of decomposable impurity functions, 
	which can be used by our algorithm as a surrogate for discriminative power.
	As a result, our algorithm generalizes traditional tree-induction algorithms, 
	such as \cart and \cfourfive, into complexity-aware methods.
	\item We provide a comprehensive experimental evaluation
	in which we show that \revise{the enhanced \cfourfive and \cart strike an excellent balance between predictive accuracy and tree complexity},
	compared to their corresponding original heuristics.
	Furthermore, the \asr formulation yields inferior predictive accuracy, compared to other learning methods.
	Our implementation is publicly available.\footnote{\url{https://github.com/Guangyi-Zhang/low-expected-cost-decision-trees}}
\end{itemize}

The rest of the paper is organized as follows.
The related work is discussed in Section \ref{section:related}.
The necessary notation and the formal definition of the \nci problem 
are introduced in Section \ref{section:definition}.
The main algorithm and its theoretical analysis follow in Sections~\ref{section:algorithm} 
and~\ref{section:analysis}, respectively.
Empirical experiments are conducted in Section \ref{section:experiments}, 
and we conclude in Section \ref{section:conclusion}.


%% file: related.tex
\section{Related work}
\label{section:related}


\para{Decision-tree induction.}
Mainstream algorithms such as \cfourfive and \cart embrace a top-down greedy approach. 
Most of the greedy criteria proposed are essentially ad-hoc heuristics for measuring the strength of dependence between tests and the class, with no consideration for tree complexity \citep{murthy1998automatic}.
Theoretical understanding about such greedy methods is still lacking in the literature.
A lower bound on the expected tree depth for \cfourfive that depends on the shape of a given tree has been developed by Bellala et al.~\citep{bellala2012group}.
There also exist some recent results in the field of learning theory \citep{brutzkus2019on,blanc2020top}.

\spara{Tree complexity.}
Popular measures include the number of nodes in the tree, the tree height and the expected path length. 
%
\revise{
The first kind of measures are closer to a notion of ``global'' interpretability, 
in the sense that one could inspect the entire tree of a small size,
while the second kind of measures provide a notion of ``local'' interpretability, 
in the sense that one could explain any given object using a small number of tests.
Our choice in the paper, the third kind of a measure, 
combines elements from both global and local interpretability.
First, it obviously enables a form of local interpretability, i.e., a guarantee of a small expected number of tests when explaining a given object.
This choice is considered to be more natural and less strict compared to worst-case tree height, 
as it may not be possible to classify every object using a small number of tests.
Second, it also enables a form of global interpretability, as the global model knowledge is acquired by understanding the decision for every example in the dataset, 
and also it leads to smaller trees in general.}
Unfortunately, these complexity measures are proven to lead to \np-hard tasks 
\citep{hancock1996lower,laurent1976constructing}. 
In particular, the expected path-length measure with an arbitrary probability distribution over objects does not admit sub-logarithmic approximation \citep{chakaravarthy2007decision}.

\spara{Identification.}
The entity identification (\entid) problem has been investigated in different contexts, 
including optimal decision trees, 
disease/fault diagnosis, and active learning~\citep{adler2008approximating,chakaravarthy2007decision,dasgupta2005analysis,garey1972optimal,guillory2009average,gupta2017approximation,kosaraju1999optimal}.
A class-based generalization, the group identification (\clsid) problem, 
where objects are partitioned into groups (classes), 
has also been studied~\citep{bellala2012group,cicalese2014diagnosis,golovin2010near}.
The state-of-the-art method achieves $\bigO(\log n)$-approximation in a general setting with an arbitrary object distribution and non-uniform multi-way testing costs~\citep{cicalese2014diagnosis}.
Our paper further generalizes the latter work by considering the discriminative power of the selected tests.
To the best of our knowledge, this is the first work
to combine identification problems and traditional tree-induction algorithms.

\spara{Stochastic submodular coverage (\stosubmcover)}
Tree induction can be seen as a sample-based stochastic 
submodular-coverage problem \citep{golovin2011adaptive,grammel2016scenario}, 
by relating a realization of items in the \stosubmcover problem to an object in identification problems.
The expected cost of a tree is then equivalent to the expected cost in item selection.

\spara{Adaptive submodular ranking (\asr)}
The \asr problem, proposed by \citet{navidi2020adaptive}, 
originates from the line of research of min-sum set cover \citep{feige2004approximating,im2012minimum}, and 
turns out to generalize the above-mentioned identification problems~\citep{bellala2012group,cicalese2014diagnosis,golovin2010near}.
Our formulation follows the framework of \asr, and extends 
its greedy criterion to incorporate an impurity-reduction component.

%% file: definition.tex
\section{Problem definition}\label{section:definition}

In this section, we first formalize the \emph{non-overfitting group identification} (\nci) problem, 
and then define a family of decomposable impurity functions for tree induction.

An instance of the \nci problem is specified by 
a set of objects $\objs=\{\obj_1,...,\obj_{\nobjs} \}$, 
a set of class labels $\clss=\{\cls_1,...,\cls_{\nclss} \}$, 
and a set of tests $\tests=\{\test_1,...,\test_{\ntests} \}$.
The objects in \objs are drawn from a probability distribution \prob, 
i.e., object \obj in \objs occurs with probability $\prob(\obj)$.
Each object $\obj\in\objs$ is associated with a class $\cls(\obj)$ in $\clss$.
A test $\test\in\tests$ performed on an object $\obj\in\objs$ returns a value $\test(\obj) \in \{1,...,\nval_\test\}$.
We assume that employing test \test incurs cost $\cost(\test)$. 
For simplicity and without loss of generality, 
we also assume that the cost function \cost takes integral values. 
A useful quantity in our later analysis is the minimum object probability $\minprob=\min_{\obj\in\objs} \prob(\obj)$.
Finally, we assume that a threshold parameter $\thr\in[0,1]$ is given as input,
which determines a stopping condition for the decision-tree construction, 
as we will see shortly. 

We write $\tree(\objs)$ to refer to a decision tree built 
to classify the objects in~\objs.
We omit the reference to the set \objs when it is clear from the context
and just write \tree.
We also write $\tree(\node)$ to refer to a subtree of the decision tree
to classify objects in a node \node of the tree,
where $\node\subseteq\objs$ is the subset of objects.
Each internal node \node is equipped with a test \test in \tests.
Objects in \node are partitioned by test \test into multiple subnodes according to their testing outcomes $\test(\obj)$.
Using this convention we refer to the root of the decision tree simply as \objs, 
that is, the complete set of objects to be classified by the tree.
Finally, we define $\prob(\node)=\sum_{\obj\in\node} \prob(\obj)$.

We stop splitting a node $\node\subseteq\objs$ in the tree \tree when either 
($i$) the node~\node is \emph{homogeneous}, i.e., all objects in \node belong to the same class, or 
($ii$) the probability $\prob(\node)$ is no greater than the threshold parameter~\thr, 
for instance, in the case of uniform \prob, the node \node has at most $\thr\nobjs$ objects.
As a surrogate for homogeneity, we adopt a function \pair over pairs of objects.
We define $\pair(\node)$ to be the number of \emph{heterogeneous} pairs of objects in the node \node, 
i.e., pairs of objects with distinct classes.
Note that $\pair(\node)=0$ when \node is homogeneous.

As a measure of complexity for a tree rooted at \objs, 
we adopt the measure of \emph{expected cost}, which we denote by $\coste(\tree(\objs))$.
In particular, 
we define $\cost(\tree,\obj)$  as the cost of evaluating an object \obj in \tree, 
which is the sum of costs of all tests that \obj goes through in \tree.
The expected cost of a tree \tree for a set of objects \objs is then defined as
$\coste(\tree(\objs)) = \sum_{\obj\in\objs} \prob(\obj) \, \cost(\tree,\obj)$.

We are now ready to define the \nci problem. 

\begin{problem}[Non-overfitting group identification (\nci)]
\label{problem:nci}
Given a problem instance $\inst = (\objs, \clss,\allowbreak \tests,\cls,\allowbreak \prob, \cost, \thr)$,
with set of objects \objs, 
set of class labels \clss, 
set of tests \tests,
object labels \cls,
probability distribution \prob, 
cost function \cost,
and a threshold~\thr,
find a tree $\tree(\objs)$ 
that minimizes the expected cost $\coste(\tree(\objs))$
and for all leaf nodes~\node it satisfies either $\pair(\node)=0$ or $\prob(\node)\le\thr$.
\end{problem}

The \nci problem generalizes the \clsid problem by setting $\thr=0$, and 
as stated in Section \ref{section:related}, the \clsid problem is \np-hard.
Thus, we aim to find a tree \tree that is an approximate solution, 
i.e., whose cost $\coste(\tree)$ is bounded with respect to the cost
$\coste(\opt{\tree})$ of the optimal~tree~$\opt{\tree}$.

Our approach draws inspiration from the \emph{adaptive submodular ranking} (\asr) problem \citep{navidi2020adaptive}, 
which can be defined similarly, 
by replacing each object $\obj_i$ in \objs with a non-decreasing submodular function $\objf_i: 2^\tests \to [0,1]$ such that $\objf_i(\emptyset)=0$ and $\objf_i(\tests)=1$;
recall that \tests is the set of tests, 
thus, each function $\objf_i$ takes as input a subset of tests.
We denote the set of non-decreasing submodular functions by $\objfs = \{\objf_i\mid \obj_i \in \objs \}$.
We again consider a tree, which recursively partitions~\objfs.
The tests \tests and the probability distribution \prob apply 
to the set of functions \objfs in the same way that they apply to their corresponding objects.
For example, a function $\objf_i$ evaluated on a test $\test\in\tests$ 
returns a value $\test(\objf_i)=\test(\obj_i)$, 
which determines the branch of the tree that $\objf_i$ will follow.
Given a tree \tree, a function~\objf picks up all tests associated with the nodes it goes through and is \emph{fully covered} when it reaches its maximum function value 
$\objf(\tests)$.
Let $\cost(\tree,\objf)$ be the cost of \emph{covering} \objf in \tree, 
defined as the sum of costs of all tests that \objf goes through in \tree before it is \emph{fully covered}.
Note that a function is not necessarily covered in a leaf node, 
it may be covered in an internal node.
The expected cost of a tree \tree is defined in a similar manner as for the \nci problem.
The \emph{adaptive submodular ranking} problem is defined as follows.

\begin{problem}[Adaptive submodular ranking (\asr)~\citep{navidi2020adaptive}]
\label{problem:asr}
Given a problem instance $\inst = (\objfs, \tests, \prob, \cost)$,
with set of submodular functions \objfs, 
set of tests \tests, 
a probability distribution \prob, 
and cost function \cost,
find a tree $\tree(\objfs)$ that covers all functions in \objfs and  
minimizes the expected cost $\coste(\tree(\objfs)) = \sum_{\objf\in\objfs} \prob(\objf) \cost(\tree,\objf)$.
\end{problem}

\para{Decomposable impurity functions.}
When constructing decision trees for classification tasks, 
in addition to having small expected cost, the discriminative power of the selected tests is also vital. 
A number of different impurity measures have been widely used in deciding a discriminative test in decision trees, 
such as \emph{entropy} and \emph{Gini index}.
Such impurity measures are defined as functions $\fimp: [0,1]^\nclss \to \realnonneg$,  
taking as input the class distribution at a given tree node.
Impurity functions are expected to satisfy certain conditions \citep{kearns1999on}, 
which capture the notion of ``impurity.''
All impurity functions mentioned in this paper satisfy the following conditions:
(1) they obtain the maximum value if the class distribution is uniform, and the minimum value zero if a node is pure (i.e., homogeneous); 
(2) they are concave; and
(3) they are symmetric.

A typical splitting criterion compares the change in impurity before and after performing a test \test, 
defined as $\fimp(\node)$ and $\fimp(\node\mid\test)$, respectively.
The \emph{impurity reduction} of a test \test on a tree node \node is defined as
$\fdiff(\node,\test) = \fimp(\node) - \fimp(\node\mid\test)$.
A test that causes larger impurity reduction is considered more discriminative.
Based on the concavity property of \fimp, it is easy to show that $\fdiff(\node,\test)\ge0$
for any tree node \node and test \test.
We defer the proof of this claim to the Appendix, Section~\ref{sec:impurity}.

Before we define a special family of impurity functions for our problem, we first introduce some additional notation.
For a node \node of the tree,
where $\node\subseteq\objs$, 
we define $\node^\val_\test$ as the child node of \node by equipping \node with test \test and following the branch that takes on a specific testing value $\val$.
In particular, we define $\node^{(i)}_\test = \node^{\val=\test(\obj_i)}_\test$.
Likewise, we define $\node^{(i)}_\testset$ as the ending node of a path that starts at $\node\subseteq\objs$ and follows a sequence of nodes each equipped with a test \test in $\testset\subseteq\tests$ by taking on a value of $\test(\obj_i)$.
Note that the order of tests in \testset does not matter in $\node^{(i)}_\testset$.
Finally, we denote the total probability of objects in a specific class \cls in \node as $\prob_\cls(\node)=\sum_{\obj\in\node:\cls(\obj)=\cls} \prob(\obj)$.

We are now ready to define $\fimp(\node)$ and $\fimp(\node\mid\test)$ for our problem.
We require \fimp to be \emph{decomposable}, i.e., to be a weighted sum over impurity scores in each class.
We define:
\begin{align}
\label{equation:impurity}
\fimp(\node) 
~=~ \sum_{\cls} \frac{\prob_\cls(\node)}{\prob(\node)} \fimp_\cls(\node) 
~=~ \frac{1}{\prob(\node)} \sum_{\obj\in\node} \prob(\obj) \fimp_{\cls(\obj)}(\node), 
\end{align}
where $\fimp_\cls(\node)$ can be any function of $\frac{\prob_\cls(\node)}{\prob(\node)}$, 
the proportion of objects of class \cls in~\node, which ensures that \fimp satisfies the three requirements stated above
(i.e., 
(1) being maximized at uniform class distribution and minimized at homogeneity, (2) concavity, and (3) symmetry).

A wide range of concave impurity functions adopt such a form.
For example, \fimp becomes the entropy function when $\fimp_\cls(\node) = -\log \frac{\prob_{\cls}(\node)}{\prob(\node)}$, and
it becomes the Gini index when $\fimp_\cls(\node) = 1-\frac{\prob_{\cls}(\node)}{\prob(\node)}$.
With the impurity of a node \node defined, $\fimp(\node\mid\test)$ is just a weighted sum of the impurity of all child nodes of \node when split by test \test, i.e., 
$\fimp(\node\mid\test) = \sum_{\val\in[\nval_\test]} \frac{\prob(\node_\test^\val)}{\prob(\node)} \fimp(\node_\test^\val)$.

A useful quantity for our analysis is the maximum value of $\fimp_{\cls(\obj)}(\node)$,
which we denote by $\minimp=\max_{\node\subseteq\objs} \max_{\obj\in\node} \fimp_{\cls(\obj)}(\node)$.

%% file: algorithm.tex
\section{Algorithm}
\label{section:algorithm}

The main idea of our approach is to cast the \nci problem as an instance of the \asr problem~\citep{navidi2020adaptive}.
We achieve this by defining a non-decreasing submodular function for each object.
The \asr problem is solved by a greedy algorithm that picks tests to maximize the coverage of the submodular functions 
while encouraging a balanced partition.
We further incorporate the impurity-reduction objective into the greedy criterion to encourage the selection of discriminative tests, without losing the approximation~guarantee.

\input{algorithm-asr}

Our algorithm for the \nci problem is demonstrated in Algorithm~\ref{alg:asr}.
It is a greedy algorithm, 
which, at each node~\node,  
selects a test \test
that maximizes a cost-benefit greedy score $\gs(\test)$ consisting of the following three terms:
\begin{align}
\gs(\test) ~=~ \frac 1{\cost(\test)} \left(
\underbrace{ \fbalance(\test) }_{\text{balance}}
+ \underbrace{ \feff(\test) }_{\text{efficiency}}
+ \,\,\paramfimp\!\!\!\! \underbrace{ \fdiscrim(\test) }_{\text{discrimination}} \right).
\label{eq:greedy_creterion}
\end{align}

The first term,
$\fbalance(\test) = \prob(\node) - \prob(\node_\test^{\opt{\val}})$, 
with $\opt{\val} = \arg\max_{\val\in[\nval_\test]} |\node_\test^\val|$,
is the sum of the branch probabilities except the largest-cardinality branch.
Maximizing $\fbalance(\test)$ encourages selecting a test \test 
that yields a balanced split.

The second term,
\eq{
\feff(\test) = \sum_{i:\obj_i\in\node} \prob(\obj_i) \frac{\fOR_i(\testset\cup\{\test\}) - \fOR_i(\testset)}{\fOR_i(\tests) - \fOR_i(\testset)},
}
is the re-weighted total sum of the marginal gain in each submodular function, 
which we will define for our objects shortly. 
Maximizing $\feff(\test)$ accelerates the progress towards termination.

The last term,
$\fdiscrim(\test) = \prob(\node) \left(\fimp(\node) - \fimp(\node\mid\test)\right)$,
is the impurity reduction we defined in Section \ref{section:definition}, 
which improves the discrimination of the selected test.
The user-defined parameter $\paramfimp\ge 0$ controls the trade-off between tree complexity and discrimination.

\revise{One way to understand the greedy score $\gs(\test)$ is to view the $\fbalance$ and $\feff$ terms as a \emph{regularizer}.}
Notice that maximizing only the first two terms, 
$\pr{\gs}(\test) = \frac{1}{\cost(\test)} \left(\fbalance(\test)+\feff(\test)\right)$ 
at Step \ref{eq:greedy-asr} of Algorithm \ref{alg:asr},
is exactly the greedy criterion used by \citet{navidi2020adaptive} to solve the \asr problem.


We finish the description of our method by showing how to define the submodular function $\objf_i$ for each object $\obj_i$.
We start by defining two monotonically non-decreasing submodular functions. 
For each object $\obj_i\in\objs$, both submodular functions take as input a subset of tests $\testset\subseteq\tests$ 
and return a real value.
The first function $\fprob_i(\testset)$ is defined as the scaled total probability of the objects 
that do not fall into $\node^{(i)}_\testset$,
i.e., the objects that disagree with $\obj_i$ in at least one test in $\testset\subseteq\tests$.
Note that eventually, only object $\obj_i$ itself stays in~$\node^{(i)}_\tests$.
Formally, we define $\fprob_i$ as 
\eq{
\fprob_i(\testset) = (1 - \prob(\node^{(i)}_\testset)) / (1-\prob(\obj_i)).
}

The second function $\fpair_i(\testset)$ is defined as the number of heterogeneous pairs 
that do not fall into $\node^{(i)}_\testset$.
Eventually, no heterogeneous pair will exist in $\node^{(i)}_\tests$ and the ending node is homogeneous.
We define 
\eq{
\fpair_{i}(\testset) = (\pair(\objs) - \pair(\node^{(i)}_{\testset})) / \pair(\objs).
}

The target (maximum) values 
for these two functions are both 1,
for each object $\obj_i$. 
Thus, the functions $\fprob_i$ and $\fpair_{i}$  
are \emph{fully covered} for a subset of tests  $\testset\subseteq\tests$
for which $\fprob_i(\testset) = 1$ and $\fpair_{i}(\testset) = 1$, respectively.
It is easy to see that both functions are submodular and monotonically non-decreasing.
When the termination constraint \thr for minimum probability is in place, 
we use the fact that the monotonicity and submodularity properties remain valid when truncated by a constant.
The truncated version of the $\fprob_i$ function is defined~as
\eq{
\bar{\fprob_i}(\testset) = \min \left\{(1 - \prob(\node^{(i)}_\testset)) / (1-\max \{\prob(\obj_i),\thr\}), 1\right\}.
}

Next we define the \emph{disjunction} function $\fOR_i$
of $\bar{\fprob_i}$ and $\fpair_i$, 
which remains monotonically non-decreasing and submodular \citep{deshpande2016approximation,guillory2011simultaneous}.
We set
\eq{
\fOR_i(\testset) = 
1 - 
\left(1 - \bar{\fprob_i}(\testset)\right) 
\left(1 - \fpair_i(\testset)\right).
}
It is easy to see that 
with a reasonable value of $\thr$ (e.g., a multiple of the greatest common divisor of $\{ \prob(\obj) \}$),
the minimum positive incremental value of any element and any $\fOR_i$ is 
\eq{
\mininc = \min_{\substack{i \in [\nobjs],\,\testset \subseteq \tests, \\ \test \in \tests: \fOR_i(\testset \cup \{\test\}) > \fOR_i(\testset)}} \left\{ \fOR_i(\testset \cup \{\test\}) - \fOR_i(\testset) \right\} = \Omega(\minprob/\nobjs^2).
}

\note{
\begin{align}
	& \fOR_i(\testset + \test) - \fOR_i(\testset) \\
	&= 
	(1 - \bar{\fprob_i}(\testset))
	(1 - \fpair_i(\testset))
	-
	(1 - \bar{\fprob_i}(\testset + \test)) 
	(1 - \fpair_i(\testset + \test)) \\
	&\ge 
	[(1 - \bar{\fprob_i}(\testset)) - (1 - \bar{\fprob_i}(\testset + \test)) ]
	(1 - \fpair_i(\testset)) \\	
	&\ge 
	[ \bar{\fprob_i}(\testset + \test) - \bar{\fprob_i}(\testset) ]
	/ {\nobjs \choose 2} \\
	&\ge 
	\Omega(\minprob) / {\nobjs \choose 2}.
\end{align}
In the last step, it is straightforward if $\thr=0$ or $\thr \le \prob(\obj_i)$ or $\bar{\fprob_i}(\testset + \test) < 1$;
otherwise, 
\[
	1 - (1 - \prob(\node^{(i)}_\testset)) / (1-\thr)
	= \Omega( \prob(\node^{(i)}_\testset) - \thr )
	= \Omega( \minprob ),
\]
as both $\prob(\node^{(i)}_\testset)$ and $\thr$ are a multiple of \minprob.
}

Last, we \revise{examine} the time complexity of Algorithm \ref{alg:asr}.
\revise{The greedy score for a tree node \node with respect to a test can be computed in $\bigO(|\node|)$ time.
At each level of the decision tree, the union of disjoint nodes has a total size \nobjs.}
Thus, the worst-case time complexity is $\bigO(H \ntests \nobjs)$, 
where \revise{$\ntests$ is the number of tests, $\nobjs$ is the number of objects, and $H$} is the tree height, which is upper bounded by $\nobjs$.
In practice, the algorithm is more efficient than what this worst-case bound suggests;
it has the same time complexity as standard tree-induction algorithms, such as \cart.

%% file: algorithm-asr.tex
\begin{algorithm}[t]
\caption{Greedy tree-induction algorithm}\label{alg:asr}
\begin{algorithmic}[1]
\Require
An instance $\inst = (\node, \clss, \tests, \cls, \prob, \cost, \thr)$, 
a set of tests $\testset\subseteq\tests$ used so far,  
impurity function~\fimp, trade-off parameter $\paramfimp\ge 0$
\Ensure
A decision tree \tree
\State {\bf Return} a decision tree \Call{Tree}{$I, \emptyset$}
\Comment{A recursive call at the top level with $\testset=\emptyset$}
\Function{Tree}{$\inst,\testset$} 
	\If{\node is homogeneous \space or \space $\prob(\node) \le \thr$}
		\State {\bf Return} a leaf node \node, labeled with its majority class
	\EndIf

	\For{$\test\in\tests\setminus\testset$}
		\State $\opt{\val} \gets \arg\max_{\val\in[\nval_\test]} |\node_\test^\val|$
		\Comment{\revise{the largest child node by test value $\opt{\val}$}}
		\State Let $\pr{\gs}(\test) = \frac 1{\cost(\test)} \big(
			\prob(\node) - \prob(\node_\test^{\opt{\val}}) +
			\sum_{i:\obj_i\in\node} \prob(\obj_i) \frac{\fOR_i(\testset\cup\{\test\}) - \fOR_i(\testset)}{\fOR_i(\tests) - \fOR_i(\testset)}
			\big)$ \label{eq:greedy-asr}
		\State Let $\gs(\test) = \pr{\gs}(\test) + \frac 1{\cost(\test)} \paramfimp \, \prob(\node) (\fimp(\node) - \fimp(\node\mid\test))$ \label{eq:greedy}
	\EndFor	
	\State $\opt{\test} \gets \arg\max_{\test \in \tests\setminus\testset} \gs(\test)$ \label{eq:greedy-step}
	
	\For{$\val\in[\nval_{\opt{\test}}]$}
		\State $\tree_{\val} \gets$ \Call{Tree}{$(\node_{\opt{\test}}^{\val}, \clss, \tests, \cls, \prob, \cost, \thr), \testset\cup\{\opt{\test}\}$}
	\EndFor
	\State {\bf Return} a tree rooted at \node with children $\{\tree_{\val}\}$ by test $\opt{\test}$
\EndFunction

\end{algorithmic}
\end{algorithm}

%% file: analysis.tex
In this section we establish the approximation guarantee of Algorithm~\ref{alg:asr}
for the \nci problem.
Our main result is the following.

\begin{theorem}
\label{thm:main}
Algorithm \ref{alg:asr} provides an $\bigO(\log 1/\minprob + \log \nobjs + \paramfimp \minimp)$ approximation guarantee for the \nci problem.
\end{theorem}

As a practical consequence of Theorem~\ref{thm:main}, we have the following corollary, 
which is a consequence of the fact that for the popular impurity functions 
the factor \minimp in the approximation ratio of Theorem~\ref{thm:main}
can be effectively bounded.
Note that in practice $\minprob \ge 1/\nobjs$ when given training data of \nobjs points, and thus 
$\bigO(\log 1/\minprob)=\bigO(\log \nobjs)$.
We omit the constant \paramfimp for simplicity.

\begin{corollary}
Algorithm \ref{alg:asr} provides an $\bigO(\log 1/\minprob + \log \nobjs)$ approximation guarantee 
for the \nci problem when the impurity function $\fimp$ is either the entropy or the Gini index function.
\end{corollary}

\begin{proof}
In addition to the result of Theorem \ref{thm:main} we can show that \minimp is small, 
compared to the other terms, or bounded by a constant.
When $\fimp$ is the entropy function, we have
\begin{align*}
\minimp 
& = \max_{\node\subseteq\objs}  \max_{\obj\in\node} \left\{ \fimp_{\cls(\obj)}(\node) \right\}
= \max_{\node\subseteq\objs} \max_{\obj\in\node} \left\{ -\log \frac{\prob_{\cls(\obj)}(\node)}{\prob(\node)} \right\} \\
& \le -\log \frac{\min_{\obj\in\objs} \prob(\obj)}{\prob(\objs)} 
= \log 1/\minprob .
\end{align*}

When $\fimp$ is the Gini index function, we have
\begin{align*}
\minimp & = \max_{\node\subseteq\objs} \max_{\obj\in\node} \left\{ \fimp_{\cls(\obj)}(\node) \right\}
= \max_{\node\subseteq\objs} \max_{\obj\in\node} \left\{ 1-\frac{\prob_{\cls(\obj)}(\node)}{\prob(\node)} \right\}
\le 1.
\end{align*}
\end{proof}


\note{A subtle trap:
\setcover can be reduced to \asr via a simple reduction, but to our problem, it requires a sophisticated reduction.
When there is only one object, our problem does nothing.\\
Old text:
This follows by a simple reduction from the minimum set-cover problem, 
by setting $\nobjs=1$ and taking the submodular function for the single object to be a set-cover function.}

Assuming $\pclass \ne \np$, the result given by Theorem \ref{thm:main} is asymptotically the best possible among instances where $1/\minprob$ is polynomial in $\nobjs$.
This follows directly from the hardness result of the \entid problem~\citep{chakaravarthy2007decision}, 
which in turn, is proved via a reduction from the minimum set-cover problem.
Recall that by specifying \thr to be zero, \nci problem degenerates into \entid or \clsid problems.
The constructed \entid instance in their reduction asks for a minimum object probability \minprob such that $1/\minprob=\bigtheta(\nobjs^3)$, and thus if \nci admits $\smallO(\log 1/\minprob) = \smallO(\log\nobjs)$ approximation, we could solve the set-cover problem with $\smallO(\log\nobjs)$-approximation, which is conditionally impossible \citep{feige1998threshold}.

\begin{remark}
The \nci problem does not admit an $\smallO(\log\nobjs)$ approximation algorithm, unless $\pclass = \np$.
\end{remark}

\subsection{Proof of Theorem \ref{thm:main}}

Our analysis is similar to the one by \citet{navidi2020adaptive}, 
except that we need a new proof of their key lemma for our new greedy selection rule (Equation~(\ref{eq:greedy_creterion})).
This is done by leveraging the special structure in the family of impurity functions (Equation~(\ref{equation:impurity})) we employ.

In order to analyze the total cost along a path, we treat cost as discrete ``time'' 
---~or continuous time if we allow continuous cost~--- and we divide time geometrically.
We refer to the decision tree returned by Algorithm \ref{alg:asr} as \treealg, 
while we refer to the optimal decision tree as \treeopt.
We denote the set of internal (i.e., unfinished) nodes up to time \tim in \treealg as $\nodesuc(\tim)$, 
and similarly as $\opt{\nodesuc}\!(\tim)$ in \treeopt.

We define
$\nodesuc_\idx = \nodesuc(\factor 2^\idx)$ and 
${\nodesuc_{\idx}^{*}} = \opt{\nodesuc}\!(2^{\idx-1})$, 
for a constant \factor to be defined shortly.
That is, we are interested in the set of unfinished nodes at the end of the $\idx$-th geometrically increasing time interval.
Notice that the interval length for $\nodesuc$ is stretched by a factor of 2\factor, compared to $\opt{\nodesuc}$.
We define $\prob(\nodesuc(\tim)) = \sum_{\node\in \nodesuc(\tim)} \prob(\node)$,
i.e., the total probability of unfinished nodes at time \tim.
Note that $\prob(\nodesuc(\tim))$ is non-increasing as \tim grows, and in the case of integral costs we have
$\prob({\nodesuc_{0}^{*}}) = \prob(\opt{\nodesuc}\!(2^{-1}))=1$, 
i.e., no test can be completed within a fractional cost.

The cost of some test may be truncated by the defined geometrical time intervals.
To denote the actual cost of a test within an interval, we first define a path $\testpath_{i\idx}$ in \treealg 
to be the sequence of tests involved within time $(\factor2^\idx,\infty)$ for each object $\obj_i$.
A test \test selected by object $\obj_i$ appears in path $\testpath_{i\idx}$ 
during time interval $(\factor2^\idx,\infty) \cap (\tim_{i,<\test},\tim_{i,<\test}+\cost(\test)]$, 
where $\tim_{i,<\test}$ is the total cost before test \test for object $\obj_i$.
The truncated cost of a test $\test\in\testpath_{i\idx}$ within that intersection is denoted by
$\cost_{i\idx}(\test)$. 
Note that $\cost_{i\idx}(\test)\le\cost(\test)$.
We denote the set of tests before test \test in path $\testpath_{i\idx}$ by $\testpath_{i\idx,<\test}$.

Our greedy algorithm is identical to the algorithm of \citet{navidi2020adaptive}
except that their greedy-selection score $\pr{\gs}$ at Step \ref{eq:greedy-asr} 
is replaced by the new score $\gs$ at Step \ref{eq:greedy}, in order to encourage impurity reduction of the selected tests.
The key in the analysis of Navidi et al.\
is to show that $\prob(\nodesuc_{\idx+1}) \le 0.2 \prob(\nodesuc_{\idx}) + 3 \prob({\nodesuc_{\idx+1}^{*}})$, 
which is proven via an intermediate value $\gs_\idx'$ defined below.
We restate their technical result here.
\begin{lemma}{\citep[Lemma 2.4,2.5]{navidi2020adaptive}}
\label{lemma:Zk-asr}
If Algorithm \ref{alg:asr} is executed using the greedy score $\pr{\gs}$ at Step \ref{eq:greedy-asr}, then
\begin{align*}
\gs_\idx' & ~\ge~ \left(\prob(\nodesuc_{\idx+1}) - 3 \prob({\nodesuc_{\idx+1}^{*}}) \right) \pr{\factor}/3, 
\text{ and}\\
\pr{\gs^\infty_\idx} & ~\le~ \prob(\nodesuc_{\idx}) \pr{\factor}/15,
\end{align*}
where 
${\gs_\idx'} = 
\sum_{\pr{\factor}2^\idx<\tim\le\pr{\factor}2^{\idx+1}} \sum_{\node\in\nodesuc(\tim)} \pr{\gs}(\test(\node))$,
$\pr{\gs^\infty_\idx} = 
\sum_{\tim>\pr{\factor}2^\idx} \sum_{\node\in\nodesuc(\tim)} \pr{\gs}(\test(\node))$,
$\pr{\factor}=15 (1+\ln 1/\mininc +\log\nobjs)$, and
$\test(\node)$ is the greedy test for node \node.
Besides, the first inequality holds regardless of the value of $\pr{\factor}$ and 
holds as long as $\test(\node)$ is a greedy test with respect to an additive score $\pr{\gs} + \fdiscrim$, where \fdiscrim can be any non-negative function;
the second inequality holds regardless of the choice of tests $\test(\node)$ in the decision tree.
\end{lemma}

Our new greedy score $\gs$ is in an additive form required above for the first inequality.
Therefore, in our case, the difficulty mainly lies in the second inequality.
We can prove that a similar lemma holds for our new greedy score.

\begin{lemma}\label{lemma:Zk}
Algorithm \ref{alg:asr} ensures the following inequalities
\begin{align*}
\gs_\idx & ~\ge~ \left(\prob(\nodesuc_{\idx+1}) - 3 \prob({\nodesuc_{\idx+1}^{*}})\right) \factor/3, 
\text{ and} \\
\gs^\infty_\idx & ~\le~ \prob(\nodesuc_{\idx}) \factor/15,
\end{align*}
where
$\gs_\idx = \sum_{\factor2^\idx<\tim\le\factor2^{\idx+1}} \sum_{\node\in\nodesuc(\tim)} \gs(\test(\node))$,
$\gs^\infty_\idx = \sum_{\tim>\factor2^\idx} \sum_{\node\in\nodesuc(\tim)} \gs(\test(\node))$,
$\factor=15 (1+\ln 1/\mininc +\log\nobjs+\paramfimp \minimp)$, and
$\test(\node)$ is the greedy test for node \node.
\end{lemma}

\begin{proof}
The first inequality is easy to show.
Notice that compared to $\gs'$, $\gs$ introduces a third term of impurity reduction \fdiscrim in Equation (\ref{eq:greedy_creterion}), 
which is always non-negative (see Section \ref{section:definition}) and thus $\gs(\test) \ge \pr{\gs}(\test)$. 
Thus, the first inequality in Lemma \ref{lemma:Zk-asr} also holds for the tree generated by the new greedy score.
Since the first inequality in Lemma \ref{lemma:Zk-asr} does not depend on the value of $\pr{\factor}$, 
we replace it with the new \factor, which completes the proof.

The second inequality requires more work.
We denote the sum of the $\fdiscrim$ terms in $\gs^\infty_\idx$ by 
\[
G = \paramfimp 
\sum_{\tim>\factor2^\idx} \sum_{\node\in\nodesuc(\tim)} 
\frac{\prob(\node)}{\cost(\test(\node))} \left(\fimp(\node) - \fimp(\node\mid\test(\node))\right).
\]
From Lemma~\ref{lemma:Zk-asr} we know that 
$\gs^\infty_\idx - G \le \pr{\factor} \prob(\nodesuc_{\idx}) /15$.

We now upper bound the additional term $G$.
We omit \node in $\test(\node)$ when it is clear from the context.
\begin{align}
G &= \paramfimp \sum_{\tim>\factor2^\idx} \sum_{\node\in\nodesuc(\tim)} \frac{\prob(\node)}{\cost(\test)} \left(\fimp(\node) - \fimp(\node\mid\test)\right) \nonumber\\
&= \paramfimp \sum_{\tim>\factor2^\idx} \sum_{\node\in\nodesuc(\tim)} \frac{\prob(\node)}{\cost(\test)} \left(\fimp(\node) - \sum_{\val\in[\nval_{\test}]} \frac{\prob(\node_{\test}^\val)}{\prob(\node)} \fimp(\node_{\test}^\val)\right)\nonumber\allowdisplaybreaks\\
&= \paramfimp \sum_{\tim>\factor2^\idx} \sum_{\node\in\nodesuc(\tim)} \frac{\prob(\node)}{\cost(\test)} 
\left(\frac{1}{\prob(\node)} \sum_{\obj\in\node} \prob(\obj) \fimp_{\cls(\obj)}(\node) 
- \sum_{\val\in[\nval_{\test}]} \frac{\prob(\node_{\test}^\val)}{\prob(\node)} \frac{1}{\prob(\node_{\test}^\val)} \sum_{\obj\in\node_{\test}^\val} \prob(\obj) \fimp_{\cls(\obj)}(\node_{\test}^\val) \right) \nonumber\allowdisplaybreaks\\
&= \paramfimp \sum_{\tim>\factor2^\idx} \sum_{\node\in\nodesuc(\tim)} \frac{1}{\cost(\test)}
\left(\sum_{\obj\in\node} \prob(\obj) \fimp_{\cls(\obj)}(\node) 
- \sum_{\val\in[\nval_{\test}]} \sum_{\obj\in\node_{\test}^\val} \prob(\obj) \fimp_{\cls(\obj)}(\node_{\test}^\val) \right) \nonumber\allowdisplaybreaks\\
&= \paramfimp \sum_{\tim>\factor2^\idx} \sum_{\node\in\nodesuc(\tim)} \sum_{\obj\in\node} \frac{\prob(\obj)}{\cost(\test)}
\left( \fimp_{\cls(\obj)}(\node) - \fimp_{\cls(\obj)}\!\left(\node_{\test}^{\test(\obj)}\right) \right) \nonumber\allowdisplaybreaks\\
&= \paramfimp \sum_{\node\in\nodesuc_\idx} \sum_{i:\obj_i\in\node} \prob(\obj_i) \sum_{\test\in\testpath_{i\idx}} \frac{\cost_{i\idx}(\test)}{\cost(\test)} 
\left( \fimp_{\cls(\obj_i)}\!\left(\node^{(i)}_{\testpath_{i\idx,<\test}}\right) - \fimp_{\cls(\obj_i)}\!\left(\node^{(i)}_{\testpath_{i\idx,<\test}\cup\{\test\}}\right) \right) \label{eq:doublecounting}\\
&\le \paramfimp \sum_{\node\in\nodesuc_\idx} \sum_{i:\obj_i\in\node} \prob(\obj_i) \sum_{\test\in\testpath_{i\idx}} 
\left( \fimp_{\cls(\obj_i)}\!\left(\node^{(i)}_{\testpath_{i\idx,<\test}}\right) - \fimp_{\cls(\obj_i)}\!\left(\node^{(i)}_{\testpath_{i\idx,<\test}\cup\{\test\}}\right) \right) \label{eq:truncatedcost}\\
&\le \paramfimp \sum_{\node\in\nodesuc_\idx} \sum_{i:\obj_i\in\node} \prob(\obj_i)\, \fimp_{\cls(\obj_i)}(\node) \label{eq:telescoping}\\
&\le \paramfimp \sum_{\node\in\nodesuc_\idx} \sum_{i:\obj_i\in\node} \prob(\obj_i)\, \minimp \nonumber\\
&= \paramfimp\, \prob(\nodesuc_\idx)\, \minimp, \nonumber
\end{align}
where step (\ref{eq:doublecounting}) 
follows by enumerating the summands in a different order, 
step (\ref{eq:truncatedcost}) is due to $\cost_{i\idx}(\test) \le \cost(\test)$, and
step (\ref{eq:telescoping}) follows by considering the telescoping series 
of the impurity reduction along a path of an object.
Putting everything together gives
\begin{align*}
\gs^\infty_\idx = G + (\gs^\infty_\idx-G) \le \paramfimp \prob(\nodesuc_{\idx}) \minimp + \prob(\nodesuc_{\idx}) \pr{\factor}/15 
= \prob(\nodesuc_{\idx}) \factor/15.
\end{align*}
\end{proof}

Next, we use another simple lemma
that provides an upper bound on the expected cost \costalg of \treealg, and 
a lower bound on the optimal cost \costopt of \treeopt.
This result is a consequence of the geometrical division of time.
For example, to obtain an upper bound for~\costalg, 
we assume that the set of unfinished nodes stays the same as $\nodesuc(\factor 2^{\idx})$ 
during the time interval $(\factor2^\idx,\factor 2^{\idx+1}]$.
Recall that $\prob(\nodesuc(\tim))$ is a non-increasing function of time~\tim.
\begin{lemma}{\citep[Lemma 2.2]{navidi2020adaptive}}
\label{lemma:bounds}
The expected cost \costalg of the tree \treealg produced by Algorithm~\ref{alg:asr},
and the cost \costopt of the optimal tree \treeopt for the \nci problem, 
satisfy the following inequalities
\begin{align*}
\costalg & ~\le~ \factor \sum_{\idx\ge0} 2^\idx \prob(\nodesuc_\idx) + \factor, \text{ and} \\
\costopt & ~\ge~ \frac 12 \sum_{\idx\ge0} 2^{\idx-1} \prob({\nodesuc_{\idx}^{*}}).
\end{align*}
\end{lemma}

We are now ready to prove our main result, Theorem~\ref{thm:main}, 
stated in Section~\ref{section:analysis}.
The proof relies on combining the results of Lemma \ref{lemma:Zk} 
with the upper and lower bounds provided by Lemma \ref{lemma:bounds}.

\begin{proof}(Theorem~\ref{thm:main})~~
From Lemma \ref{lemma:Zk}, we obtain
\[
\left(\prob(\nodesuc_{\idx+1}) - 3 \prob({\nodesuc_{\idx+1}^{*}})\right) \factor/3 ~\le~ 
\gs_\idx ~\le~ \gs^\infty_\idx ~\le~ \prob(\nodesuc_{\idx}) \factor/15.
\]

By rearranging terms, we get 
\[
\prob(\nodesuc_{\idx+1}) ~\le~ 0.2\, \prob(\nodesuc_{\idx}) + 3\, \prob({\nodesuc_{\idx+1}^{*}}).
\]

Define $Q=\factor \sum_{\idx\ge0} 2^\idx \prob(\nodesuc_\idx) + \factor$, i.e., the upper bound of \costalg.
We have 
\begin{align*}
Q 
&= \factor \sum_{\idx\ge1} 2^\idx \prob(\nodesuc_\idx) + \factor\left(\prob(\nodesuc_0)+1\right) \\
&\le \factor \sum_{\idx\ge1} 2^\idx \left( 0.2\, \prob(\nodesuc_{\idx-1}) + 3\, \prob({\nodesuc_{\idx}^{*}})\right) + \factor\left(\prob(\nodesuc_0)+1\right) \\
&\le \factor \sum_{\idx\ge0} 2^{\idx}\, 0.4\, \prob(\nodesuc_{\idx}) 
+ \factor \frac 12\, \sum_{\idx\ge1} 2^{\idx-1}\, 12\,\prob({\nodesuc_{\idx}^{*}}) 
+ \factor\left(\prob(\nodesuc_0)+1\right) \\
&= \factor \sum_{\idx\ge0} 2^{\idx}\, 0.4\, \prob(\nodesuc_{\idx}) 
+ \factor \frac 12 \sum_{\idx\ge0} 2^{\idx-1}\, 12\,\prob({\nodesuc_{\idx}^{*}}) - 3\,\factor\,\prob({\nodesuc_{0}^{*}})
+ \factor\left(\prob(\nodesuc_0)+1\right) \\
&\le \factor \sum_{\idx\ge 0} 2^{\idx}\, 0.4 \prob(\nodesuc_{\idx}) 
+ \factor \frac 12\, \sum_{\idx\ge 0} 2^{\idx-1}\, 12\,\prob({\nodesuc_{\idx}^{*}}) \\
&\le 0.4\, Q + 12\,\factor\, \costopt,
\end{align*}
where we note that $\prob(\opt{\nodesuc_{0}})=1$ and $\prob(\nodesuc_0)\le1$.
Together with Lemma \ref{lemma:bounds}, we obtain 
\[
\costalg ~\le~ Q ~\le~ \frac{12}{0.6} \, \factor\, \costopt ~=~ 20\, \factor\, \costopt.
\]
\end{proof}

\note{Precisely, $300 (1 + \log 1/\minprob + \log \nobjs + \paramfimp \minimp)$-approximation.}

%% file: experiment.tex
\section{Experimental evaluation}
\label{section:experiments}

In this section, we evaluate the performance of our enhanced decision-tree algorithms by comparing them against strong baselines on a large collection of real-world datasets.
Some additional experimental results are presented in the Appendix, 
including 
further experimental results for \cart (Section~\ref{sec:cart}),
\revise{further experimental results on tree size (Section~\ref{sec:treesize}),}
additional statistical tests (Section~\ref{sec:tests}), 
and more visual examples (Section~\ref{sec:visual}).
Our implementation and pre-processing scripts can be found 
in a Github repository.\footnote{\url{https://github.com/Guangyi-Zhang/low-expected-cost-decision-trees}}

\spara{Datasets.}
We evaluate our methods on 20 datasets from the UCI Machine Learning Repository~\citep{UCI} and OpenML~\citep{OpenML2013}.
Information about the datasets is shown in Table \ref{tbl:datasets}. 
We experiment with datasets containing up to 0.6 million objects and 5 thousand features.
We set the limit~\thr to be 0.005 for all datasets except for small ones, whose \thr are set accordingly so that the minimum leaf size is 2.
For all datasets, 70\% of the data points are used for training, 10\% for validation and the rest for testing.
Numerical features are categorized into multiple bins \revise{by the $k$-means strategy, which can adapt to uneven data distributions}.
All categorical features are then binarized \revise{to avoid biases towards features with a large number of levels \citep{strobl2007bias}}.
All identical objects are coalesced into a single object, 
and the sampling probability is set accordingly.
To fulfill the realizability assumption, 
the majority class is assigned to each identical data point in the training set, 
which may have different classes otherwise, due to noise or feature discretization.
Apart from the original datasets with unit test cost, 
we additionally create more challenging scenarios, 
where each test cost is independently drawn from the set $\{1,\ldots,10\}$.

\begin{table}[t]
  \caption{Datasets statistics; 
  $n, \ntests, \nclss$: number of data points, binary features and classes.
  Numerical features are categorized into 5 bins by the $k$-means strategy.}
  \label{tbl:datasets}
  \centering
\begin{small}
\begin{tabular}{lrrr}
\toprule
Dataset & \multicolumn{1}{c}{$n$} 
        & \multicolumn{1}{c}{\ntests}
        & \multicolumn{1}{c}{\nclss} \\
\midrule
iris              & 150    & 20   & 3 \\
ilpd              & 583    & 46   & 2 \\
breast-w          & 699    & 45   & 2 \\
tic-tac-toe       & 958    & 27   & 2 \\
obesity           & 2\,111   & 58   & 7 \\
bioresponse       & 3\,751   & 5\,333 & 2 \\
spambase          & 4\,601   & 285  & 2 \\
phoneme           & 5\,404   & 25   & 2 \\
musk              & 6\,598   & 830  & 2 \\
speed-dating      & 8\,378   & 733  & 2 \\
phishing-websites & 11\,055  & 46   & 2 \\
shoppers          & 12\,330  & 454  & 2 \\
letter            & 20\,000  & 80   & 26 \\
default           & 30\,000  & 112  & 2 \\
bank-marketing    & 45\,211  & 76   & 2 \\
electricity       & 45\,312  & 42   & 2 \\
firewall          & 65\,532  & 55   & 4 \\
dota2             & 92\,649  & 394  & 2 \\
diabetic          & 101\,766 & 264  & 3 \\
covertype         & 581\,012 & 94   & 7 \\
\bottomrule
\end{tabular}
\end{small}
\end{table}

\spara{Algorithms and baselines.}
\revise{A summary of the algorithms is displayed in Table~\ref{tbl:algs}.}
The algorithms that implement the proposed approach are denoted as  
\emph{enhanced \cfourfive} (\ecfourfive) and \emph{enhanced \cart} (\ecart).
Baselines include the following:
\begin{itemize}
\item The \asr method \citep{navidi2020adaptive}, 
which is the greedy algorithm without the newly-introduced impurity-reduction term.
\item Impure Pairs (\ip), which maximizes the reduction in the number of impure pairs at each split, 
i.e., the unweighted edge cut among different classes \citep{golovin2010near,cicalese2014diagnosis}.
\item \bal, which is an unsupervised balanced-tree algorithm that greedily selects the test that splits the current node most evenly. 
\item The two traditional algorithms \cfourfive and \cart, 
and their cost-benefit versions that select a test using a cost-weighted criterion (denoted with a prefix `C').
\end{itemize}

\begin{table}[t]
  \caption{\revise{Summary of competing algorithms.
  Prefix `C-' indicates a variant with a cost-weighted criterion.
  Prefix `p-' indicates a variant with post-pruning.
  }}
  \label{tbl:algs}
  \centering
\begin{small}
\begin{tabular}{ll}

\toprule
Algorithm & Brief description \\
\midrule
\asr              &greedy without impurity reduction \citep{navidi2020adaptive} \\
\ip              &greedy in reducing the number of impure pairs \citep{golovin2010near} \\
\bal              &greedy in the most balanced split \\
{[p]}[C]\cart             &traditional \cart \citep{breiman1984classification}\\
{[p]}[C]\cfourfive        &traditional \cfourfive \citep{quinlan1993c4.5} \\
{[p]}\ecart               &enhanced \cart  (proposed method) \\
{[p]}\ecfourfive          &enhanced \cfourfive (proposed method) \\

\bottomrule
\end{tabular}
\end{small}
\end{table}

To ensure a meaningful comparison, we measure performance for all methods based on the same stopping criteria.
All algorithms perform two-way splitting.
Splitting of tree nodes stops if homogeneity is achieved or if the minimum-probability limit is reached.
We examine the performance of \cfourfive and \cart with post-pruning (denoted with a prefix `p') or without.
We adopt the standard \emph{minimal cost-complexity pruning} approach \citep{breiman1984classification}, 
which prunes a tree node having many leaves  
if its impurity is no much larger than the total impurity of its leaves.
The parameter that controls the stringency of the pruning 
is determined by cross-validation over a logspace from $10^{-5}$ to 1.

\begin{figure}[t]
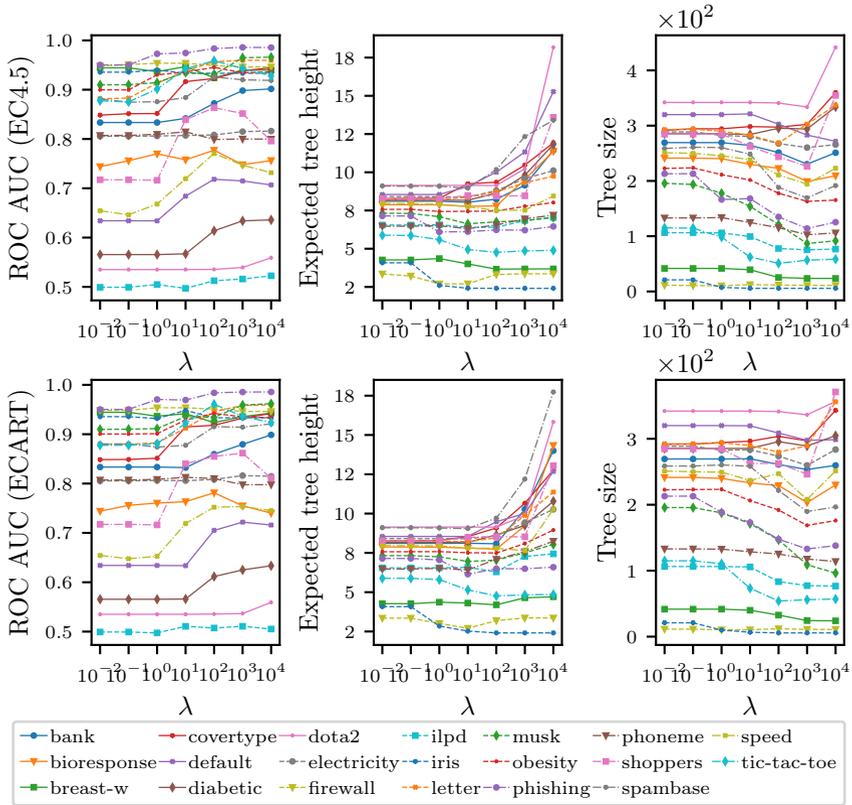

    \centering
	\pgfinput{tradeoff}
    \caption{\label{fig:tradeoff}The effect of the trade-off hyperparameter $\paramfimp$.
    \revise{ROC AUC is on validation data, and expected tree height or tree size on training data.}}
\end{figure}

The only hyperparameter in our algorithm (\paramfimp) controls the trade-off between complexity and discrimination.
The effect of \paramfimp is summarized in Figure \ref{fig:tradeoff}.
For large values of \paramfimp
our algorithms turn into the tra\-di\-tional tree-induction algorithms \cfourfive and \cart; 
on the other hand, if \paramfimp is zero, our algorithms turn into the greedy algorithm for the \asr problem.
As we are working with a bi-criteria optimization problem, 
there is no golden rule in deciding the best value of the hyperparameter.
In this experiment, we aim to decide a value of \paramfimp that preserves comparable accuracy while reduces the complexity as much as possible.
Thus, we tune the hyperparameter \paramfimp by starting with a large \paramfimp and gradually decreasing it before a significant drop (larger than 1\%) is seen in the predictive accuracy over the validation set.
Note also that \paramfimp is invariant to the data size, 
as the greedy score only depends on the distributions before and after the split.

\begin{figure}[t]
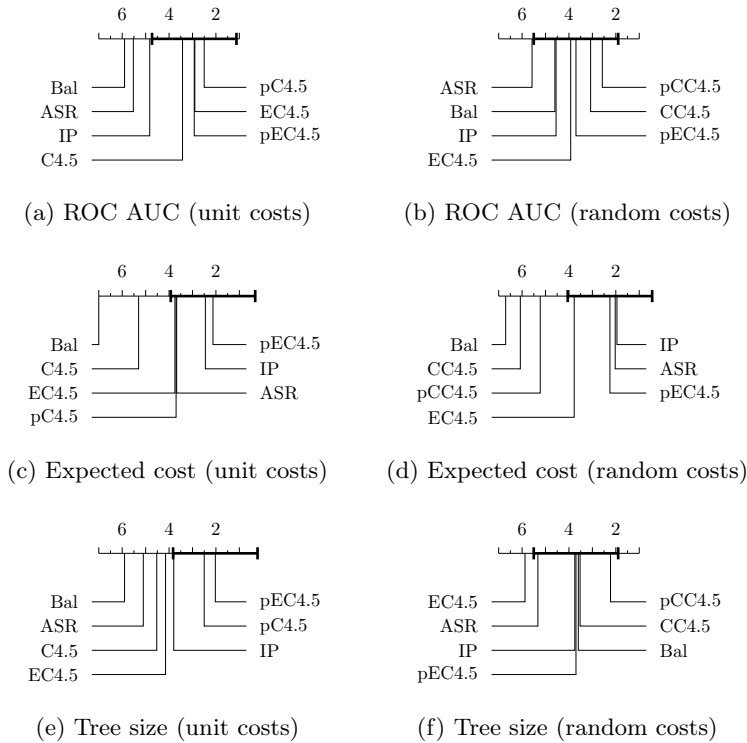

    \centering
    \subcaptionbox{\rocauc (unit costs)}{
	    \tikzinput{cd-bonferroni-auc-unit-c45-005}
    }
    \hskip -5ex
	\subcaptionbox{\rocauc (random costs)}{
        \tikzinput{cd-bonferroni-auc-random-c45-005}
    }
    \subcaptionbox{Expected cost (unit costs)}{
    	\tikzinput{cd-bonferroni-cost-unit-c45-005}
    }
    \hskip -5ex
    \subcaptionbox{Expected cost (random costs)}{
        \tikzinput{cd-bonferroni-cost-random-c45-005}
    }
    \subcaptionbox{\revise{Tree size (unit costs)}}{
	    \tikzinput{cd-bonferroni-size-unit-c45-005}
    }
    \hskip -5ex
    \subcaptionbox{\revise{Tree size (random costs)}}{
		\tikzinput{cd-bonferroni-size-random-c45-005}
    }
    \caption{\label{fig:cd:c45}Critical difference for the Bonferroni-Dunn test on significance level $\alpha=0.05$ for average ranks of algorithms among 20 tested datasets.
    Methods closer to the right end have a better rank.
    The method that is compared with other methods is p\ecfourfive, and methods lying outside the thick interval are significantly different from p\ecfourfive.
    }
\end{figure}

\begin{figure}[t]
    \centering
    \subcaptionbox{ROC AUC (unit costs)}{
        \tikzinput{cd-bonferroni-auc-unit-cart-005}
    }
    \hskip -5ex
    \subcaptionbox{ROC AUC (random costs)}{
        \tikzinput{cd-bonferroni-auc-random-cart-005}
    }
    \subcaptionbox{Expected cost (unit costs)}{
        \tikzinput{cd-bonferroni-cost-unit-cart-005}
    }
    \hskip -5ex
    \subcaptionbox{Expected cost (random costs)}{
        \tikzinput{cd-bonferroni-cost-random-cart-005}
    }
    \subcaptionbox{\revise{Tree size (unit costs)}}{
        \tikzinput{cd-bonferroni-size-unit-cart-005}
    }
    \hskip -5ex
    \subcaptionbox{\revise{Tree size (random costs)}}{
        \tikzinput{cd-bonferroni-size-random-cart-005}
    }
    \caption{\label{fig:cd:cart}Critical difference for the Bonferroni-Dunn test on significance level $\alpha=0.05$ for average ranks of algorithms among 20 tested datasets.
    Methods closer to the right end have a better rank.
    The method that is compared with other methods is p\ecart, and methods lying outside the thick interval are significantly different from p\ecart.
    }
\end{figure}

\begin{figure}[t]
    \centering
	\pgfinput{comparison-C45-unit}
    \caption{\label{fig:comparison:unitcost:c45}Performance results with unit test costs.
    All plots in the same row share the same x- and y-axes.
    Error bars are also shown.}
\end{figure}

\spara{Results.}
We evaluate all methods using \rocauc as a measure of predictive power,
expected cost as a measure of tree complexity,
\revise{and tree size (i.e., the number of tree nodes) as an auxiliary measure of global tree complexity.
A full result on tree size is deferred to Section~\ref{sec:treesize} in Appendix.}
Reported results, shown in Figure~\ref{fig:comparison:unitcost:c45},  
are averages over 5 executions with random train-test splits.
We conduct the Bonferroni-Dunn test with significance level $\alpha=0.05$ for average ranks~\citep{demvsar2006statistical},
and report the \emph{critical difference diagram} in Figure~\ref{fig:cd:c45} and \ref{fig:cd:cart},
where the proposed method p\ecfourfive (or p\ecart) is tested against the other methods, and methods closer to the right end have a better rank.
We see that the predictive power and tree complexity of p\ecart and p\ecfourfive are statistically not significantly different from the respective best performer, while it is significantly better than most other baselines.
Two methods \cfourfive and \cart lead to similar behavior; we focus on \cfourfive below and discuss its results in details.
Full results for \cart and its enhancements are presented in the Appendix, Section~\ref{sec:cart}.

It can be seen that post-pruning has a noticeable positive effect on both the accuracy and complexity for the \cfourfive algorithm.
However, even after post-pruning, p\cfourfive is still ranked closely to un-pruned \ecfourfive in terms of the expected cost, and in some datasets, the expected cost of p\cfourfive is about two times larger than that of \ecfourfive in Figure~\ref{fig:comparison:unitcost:c45}.
This is reasonable because post-pruning mainly removes tree nodes near the bottom, but fails to rescue early bad splits near the root.
\revise{On the other hand, post-pruning is significantly more beneficial than impurity reduction for the global tree size.}
Also note that post-pruning has less effect on \ecfourfive in terms of accuracy, which indicates that un-pruned \ecfourfive alone is robust to overfitting.

The decision tree produced by \bal is the worst in both aspects.
This is expected for predictive power as \bal is an unsupervised method, but it is quite surprising for complexity.
It turns out that \bal often has to keep expanding a balanced tree until the minimum leaf size is reached, as tree nodes rarely achieve homogeneity.
This behavior reinforces the argument that discriminative tests help accelerating termination and reducing expected cost.

The \ip algorithm achieves better performance in both aspects than the \asr algorithm.
However, \ip has a too strong bias towards a balanced split, that it favors a random test over a discriminative one in the example we provide in Section~\ref{section:supplementary:asr-splits}.
This bias is also reflected in Figure~\ref{fig:comparison:unitcost:c45} where it falls behind \ecart by more than 10\% accuracy in some datasets.
By further statistical tests we conduct in Section~\ref{sec:tests},  
the predictive power of \asr and \ip are statistically indistinguishable from the unsupervised \bal.

Finally, regarding running time, 
algorithm \ecfourfive typically runs 3-4 times longer than \cfourfive, 
but there are instances that the latter algorithm constructs very skewed trees and it takes more time to complete  (details in Appendix, Section~\ref{sec:time}).

\begin{figure}[t]
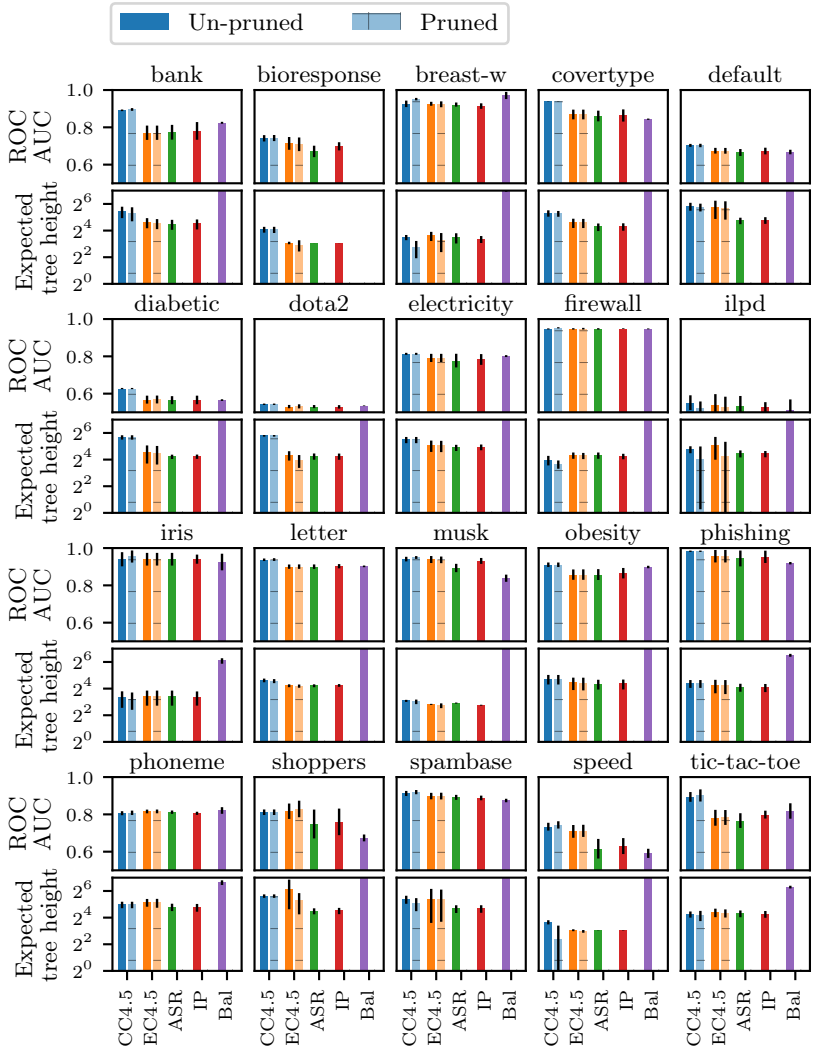

    \centering
    \pgfinput{comparison-C45-random}
    \caption{\label{fig:comparison:nonuniform:c45}Performance results with non-uniform test costs.}
\end{figure}

The benefit of the proposed method becomes more pronounced 
in non-uniform-cost scenarios, shown in Figure \ref{fig:comparison:nonuniform:c45}.
It turns out that the cost-benefit traditional trees fail to reduce the expected cost, which indicates the need for more sophisticated techniques like ours to tackle non-uniform costs.
Our algorithms obtain comparable predictive power, while achieving up to 90\% lower expected cost than the traditional trees.

We conclude that \revise{our enhancement, given in the form of a regularizer}, strikes an excellent balance between predictive power and expected tree height.

%% file: conclusion.tex
\section{Conclusion}\label{section:conclusion}

In this paper, we proposed a novel algorithm to construct a general decision tree 
with asymptotically tight approximation guarantee on expected cost under mild assumptions.
The algorithm can be used to assimilate many existing standard impurity functions
so as to enhance their corresponding splitting criteria with a complexity guarantee. 
Through empirical evaluation on various datasets and scenarios, we verified the effectiveness of our algorithm both in terms of accuracy and complexity.
Potential future directions include the study of different complexity measures, 
further termination criteria, 
and incorporating a broader family of impurity functions.


%% file: supp.tex

\section{Split criteria for adaptive submodular ranking (ASR) method may lead to non-discriminative decision trees}
\label{section:supplementary:asr-splits}

\begin{figure}[H] 
\centering
\subcaptionbox{\label{fig:impurepair:a}A discriminative split}{
	\picinput[width =.33\textwidth]{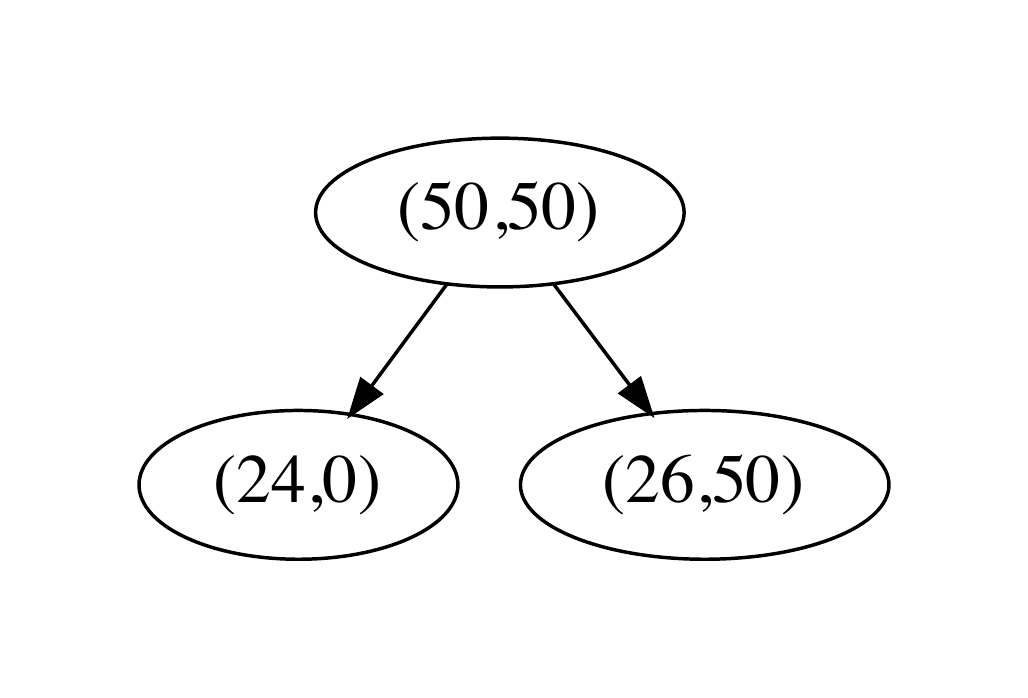}
   	}
\quad
\subcaptionbox{\label{fig:impurepair:b}A non-discriminative split}{
   	\picinput[width =.33\textwidth]{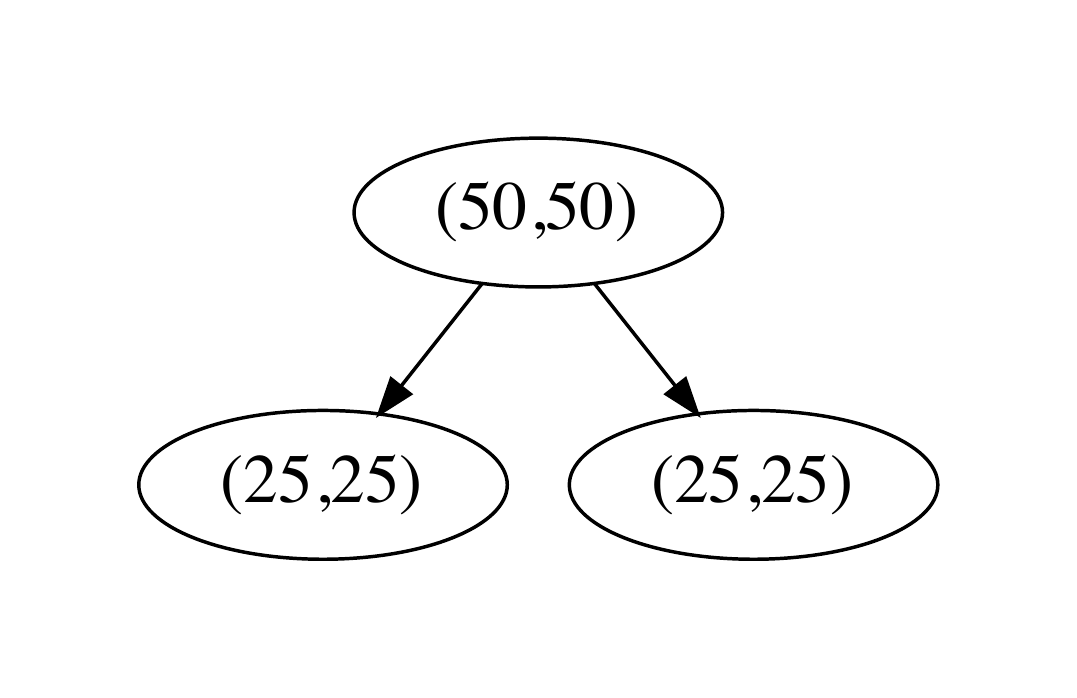}
	}
\caption{\label{fig:impurepair}
A simple example demonstrating that \asr split criteria may lead to non-discriminative decision trees.}
\end{figure}

We present a simple example demonstrating that the \asr method, 
used for the group identification problem where the aim is to minimize the expected cost of a decision tree,
may not select discriminative splits, 
which in turn may lead to decision trees with low predictive power.
The splitting criterion of \asr consists of two terms (see Section~\ref{section:algorithm} for more details).
One term encourages balanced partitions and does not use label information.
For the other term, two criteria have been considered in the literature:
(1) maximize reduction in the number of \emph{heterogeneous pairs} of objects after a split, or 
(2) maximize the number of excluded objects in other classes. 
These two criteria turn out to be equivalent up to a constant factor of 2 by a simple double counting.
Apparently, when the number of heterogeneous pairs drops to zero or each object separates from all objects in other classes, 
we obtain perfect accuracy (in training data). 
Note that even random splits can lead to perfect accuracy as long as the tree is fully expanded.
Here we discuss the second term, as random splits actually optimize balance, on expectation.

Our example is shown in Figure~\ref{fig:impurepair}, where we demonstrate two possible splits on a node.
As we will see, \asr favors the non-discriminative split \ref{fig:impurepair:b} over the more discriminative split \ref{fig:impurepair:a}.
We assume two classes, and we write $(x,y)$ to indicate the number of objects from the two different classes
in a tree node.
We now discuss the first criterion.
In the left tree \ref{fig:impurepair:a}, the root, the left, and the right node have 
$50\times50$, $0$, and $26\times50$ heterogeneous pairs, respectively.
In the right tree \ref{fig:impurepair:b}, there are $50\times50$, $25\times25$, and $25\times25$ pairs, respectively.
The left split decreases the number of heterogeneous pairs by $50\times50-0-26\times50=24\times50$.
The right split decreases the number of heterogeneous pairs by $50\times50-25\times25-25\times25=50\times25$.
Thus, the \asr criterion will select the (non-discriminative) split \ref{fig:impurepair:b}.

\section{Impurity reduction is non-negative}
\label{sec:impurity}

By the concavity property of \fimp, it is easy to show that the impurity-reduction function 
$\fdiff(\node,\test)$ is non-negative, for any tree node \node and test \test.
In particular, we have
\begin{align}
\fdiff(\node,\test) & ~=~ \fimp(\node) - \fimp(\node\mid\test) \nonumber\\
& ~=~ \fimp(\node) - \sum_{\val\in[\nval_\test]} \frac{\prob(\node_\test^\val)}{\prob(\node)} \fimp(\node_\test^\val) \nonumber\\
& ~=~ \fimp(\prob_\node) - \sum_{\val\in[\nval_\test]} \frac{\prob(\node_\test^\val)}{\prob(\node)} \fimp(\prob_{\node_\test^\val})\nonumber\\
&~\ge~ \fimp(\prob_\node) - \fimp\!\left(\sum_{\val\in[\nval_\test]} \frac{\prob(\node_\test^\val)}{\prob(\node)} \prob_{\node_\test^\val}\right) \label{inequality:concavity} \\
&~=~ \fimp(\prob_\node) - \fimp(\prob_\node) \nonumber \\
&~=~ 0 \nonumber,
\end{align}
where $\prob_\node$ denotes the class distribution vector of \node, and Inequality~(\ref{inequality:concavity}) is by concavity.

%

\section{Further experimental results for \cart}\label{sec:cart}

\begin{figure}[H]
    \centering
	\pgfinput{comparison-CART-unit}
    \caption{\label{fig:comparison:unitcost:cart}Performance results with unit test costs.
    All plots in the same row share the same x- and y-axes.
    Error bars are also shown.}
\end{figure}
\begin{figure}[H]
    \centering
    \pgfinput{comparison-CART-random}
    \caption{\label{fig:comparison:nonuniform:cart}Performance results with non-uniform test costs.}
\end{figure}

\section{Further experimental results on tree size}\label{sec:treesize}
\begin{figure}[H]
    \centering
	\pgfinput[width=.9\textwidth]{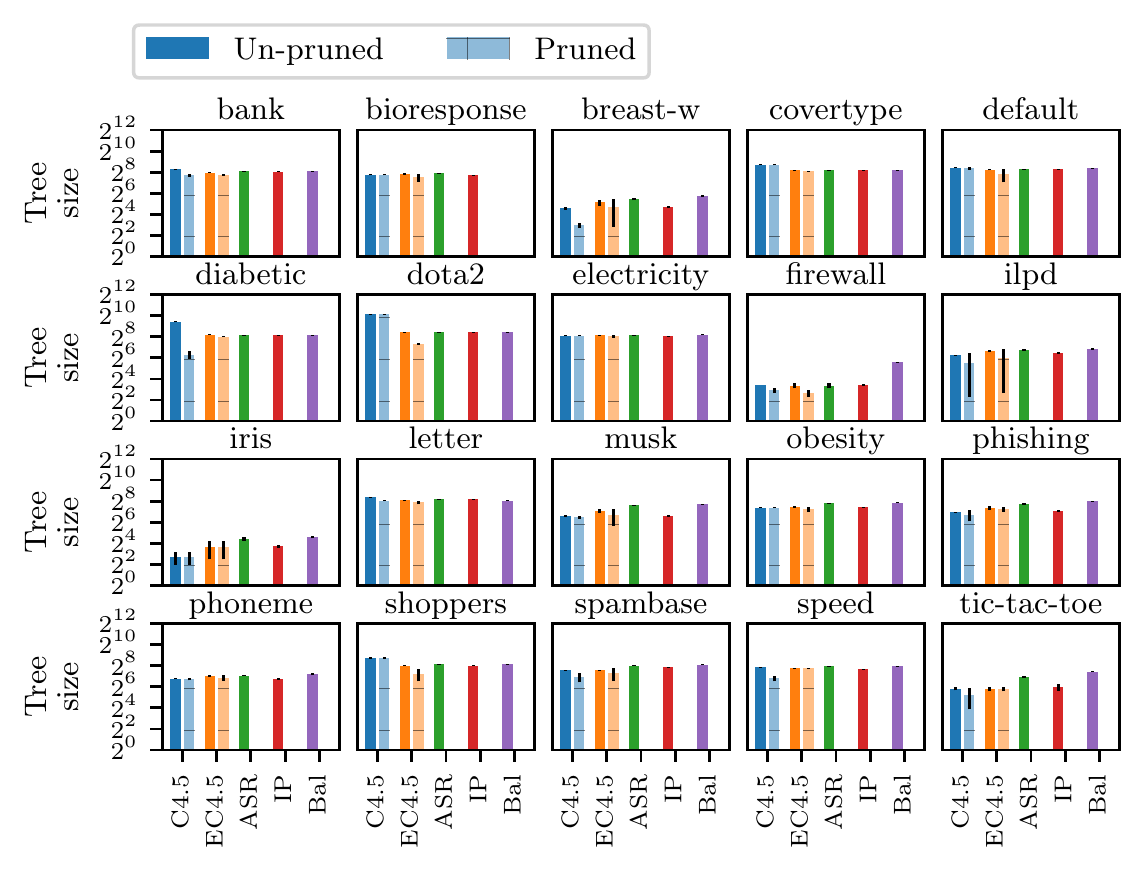}
	\pgfinput[width=.9\textwidth]{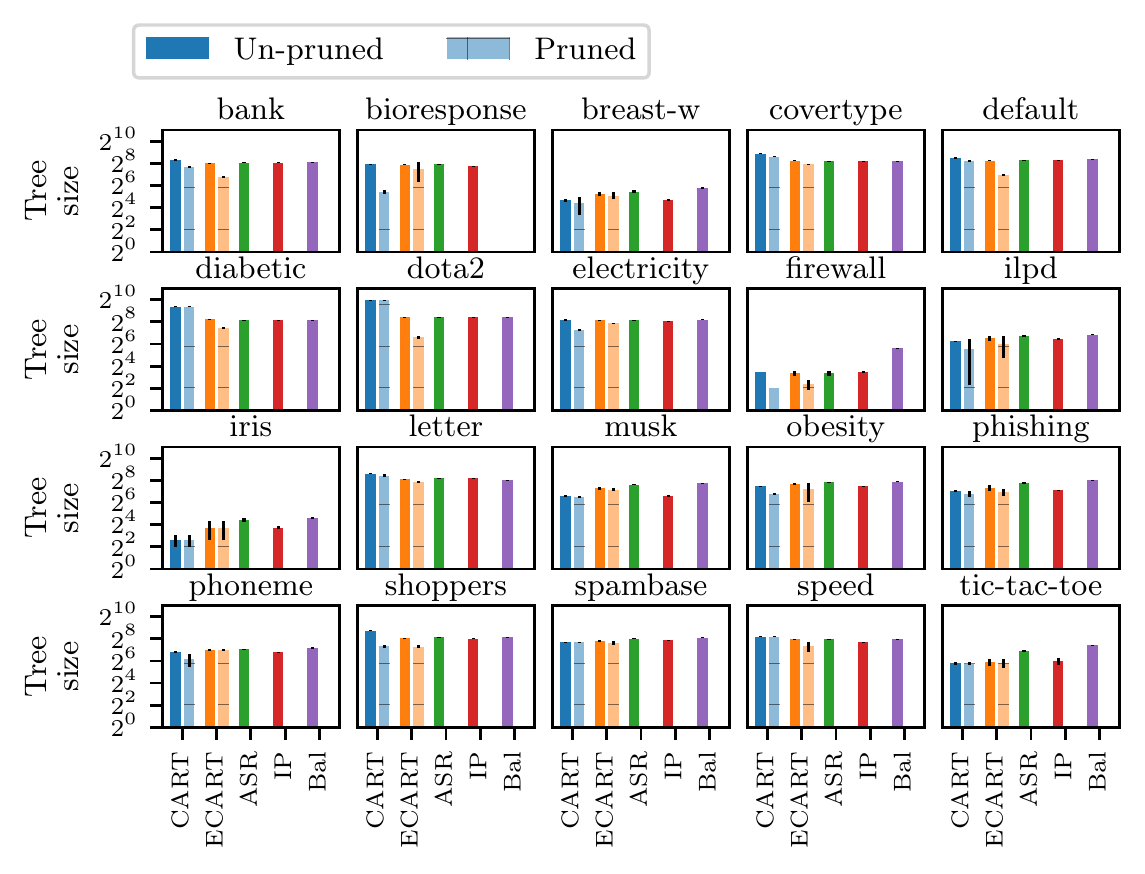}
    \caption{\label{fig:comparison:size:unitcost}Performance results with unit test costs.
    All plots in the same row share the same x- and y-axes.
    Error bars are also shown.}
\end{figure}

\begin{figure}[H]
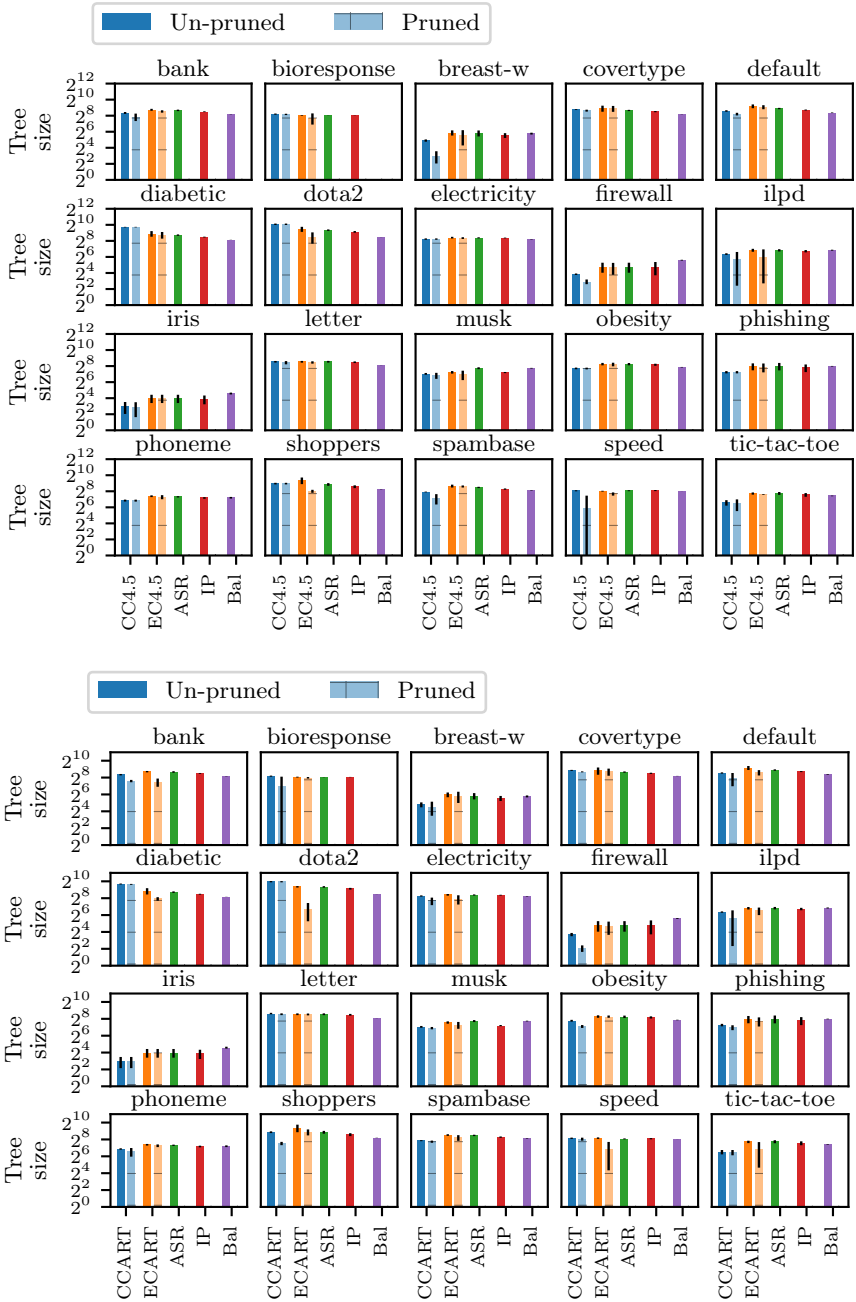

    \centering
    \pgfinput{comparison-size-C45-random}
    \pgfinput{comparison-size-CART-random}
    \caption{\label{fig:comparison:size:nonuniform}Performance results with non-uniform test costs.}
\end{figure}

\section{Further statistical tests: pairwise Nemenyi}\label{sec:tests}

\begin{figure}[H]
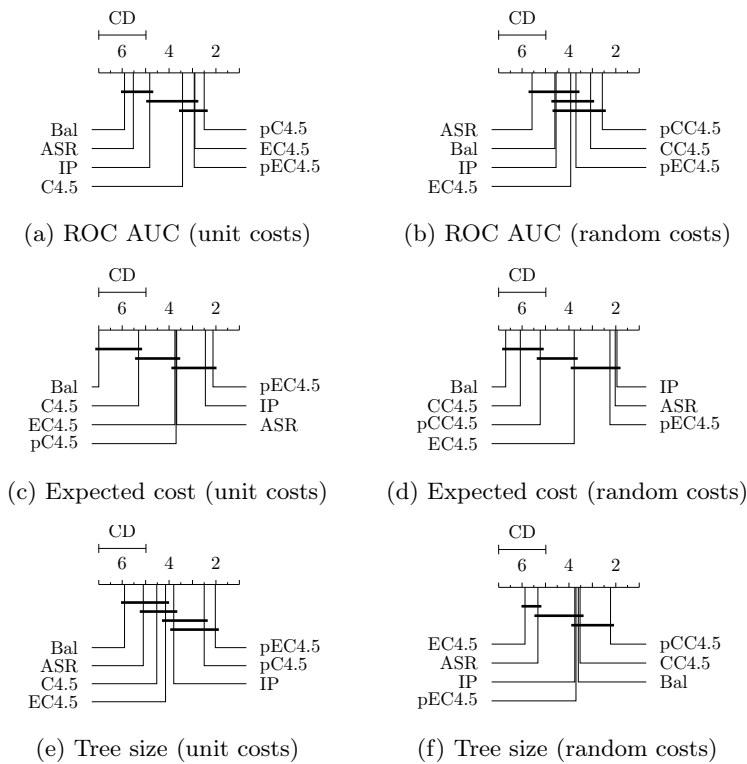

    \centering
    \subcaptionbox{ROC AUC (unit costs)}{
        \tikzinput{cd-nemenyi-auc-unit-c45-005}
    }
    \hskip -5ex
    \subcaptionbox{ROC AUC (random costs)}{
        \tikzinput{cd-nemenyi-auc-random-c45-005}
    }
    \subcaptionbox{Expected cost (unit costs)}{
        \tikzinput{cd-nemenyi-cost-unit-c45-005}
    }
    \hskip -5ex
    \subcaptionbox{Expected cost (random costs)}{
        \tikzinput{cd-nemenyi-cost-random-c45-005}
    }
    \subcaptionbox{Tree size (unit costs)}{
        \tikzinput{cd-nemenyi-size-unit-c45-005}
    }
    \hskip -5ex
    \subcaptionbox{Tree size (random costs)}{
        \tikzinput{cd-nemenyi-size-random-c45-005}
    }
    \caption{\label{fig:cd:nemenyi:c45}Critical difference for the Nemenyi test on significance level $\alpha=0.05$ for average ranks of algorithms among 20 tested datasets.
    Methods closer to the right end have a better rank.
    Any pair of methods which are not connected with an horizontal line have an average rank that is different with statistical significance.
    }
\end{figure}

\begin{figure}[H]
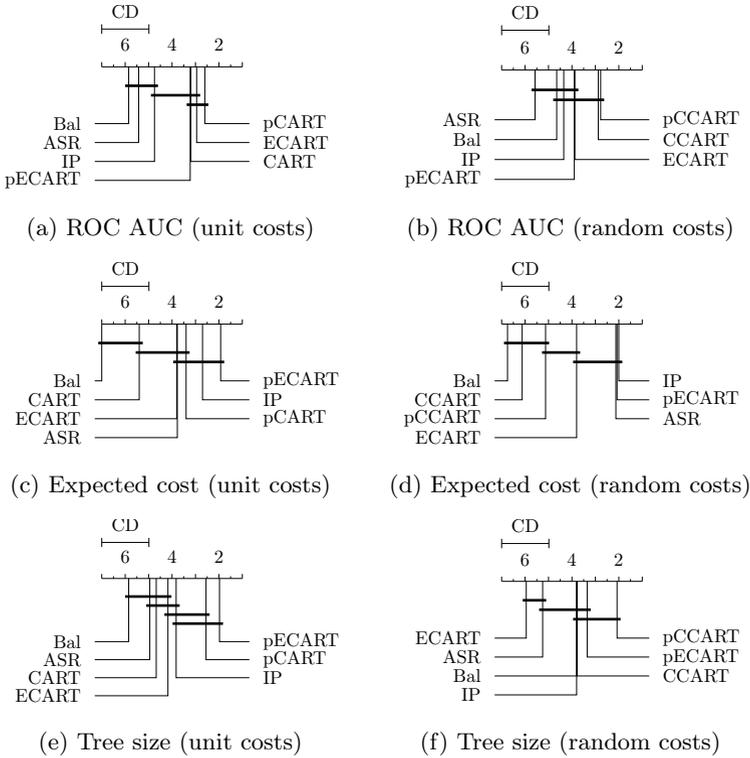

    \centering
    \subcaptionbox{ROC AUC (unit costs)}{
        \tikzinput{cd-nemenyi-auc-unit-cart-005}
    }
    \hskip -5ex
    \subcaptionbox{ROC AUC (random costs)}{
        \tikzinput{cd-nemenyi-auc-random-cart-005}
    }
    \subcaptionbox{Expected cost (unit costs)}{
        \tikzinput{cd-nemenyi-cost-unit-cart-005}
    }
    \hskip -5ex
    \subcaptionbox{Expected cost (random costs)}{
        \tikzinput{cd-nemenyi-cost-random-cart-005}
    }
    \subcaptionbox{Tree size (unit costs)}{
        \tikzinput{cd-nemenyi-size-unit-cart-005}
    }
    \hskip -5ex
    \subcaptionbox{Tree size (random costs)}{
        \tikzinput{cd-nemenyi-size-random-cart-005}
    }
    \caption{\label{fig:cd:nemenyi:cart}Critical difference for the Nemenyi test on significance level $\alpha=0.05$ for average ranks of algorithms among 20 tested datasets.
    Methods closer to the right end have a better rank.
    Any pair of methods which are not connected with an horizontal line have an average rank that is different with statistical significance.
    }
\end{figure}

\section{Running time}
\label{sec:time}
Note that algorithm \bal misses results over some datasets because its running time is too long.

All experiments were carried out on a server equipped with 24 processors of AMD Opteron(tm)
Processor 6172 (2.1 GHz), 62GB RAM, running Linux~2.6.32-754.35.1.el6.x86\_64.
We use Python~3.8.5.

\begin{figure}[H]
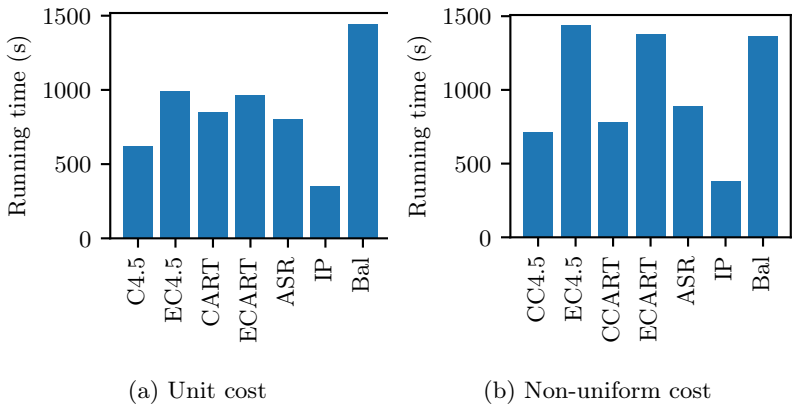

    \centering
    \subcaptionbox{Unit cost}{
        \pgfinput{running-time-unit-avg}
    }
    \hskip -5ex
    \subcaptionbox{Non-uniform cost}{
        \pgfinput{running-time-random-avg}
    }
    \caption{\label{fig:time}Running time, average over all datasets.
    }
\end{figure}

\revise{
We also demonstrate the running time for two selected datasets with unit costs, with a large number of data objects and features, to explore the impact of number of objects and features to the running time.
In general, as reflected in the worst-case time complexity $\bigO(H \ntests \nobjs)$,
the algorithms complete their computation quickly in the case of a large number of data objects (\nobjs), \emph{or} large number of features (\ntests), but not both.
Furthermore, the dependency on \nobjs is slightly worse than on \ntests, as the tree height $H$ typically has a logarithmic dependence on \nobjs.}
\begin{figure}[H]
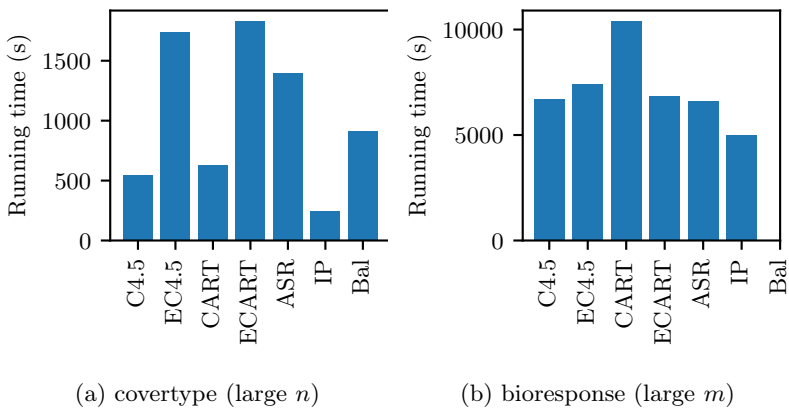

    \centering
    \subcaptionbox{covertype (large \nobjs)}{
        \pgfinput{running-time-unit-covertype}
    }
    \hskip -5ex
    \subcaptionbox{bioresponse (large \ntests)}{
        \pgfinput{running-time-unit-bioresponse}
    }
    \caption{\label{fig:time2}Running time on selected datasets.
    }
\end{figure}

\newpage
\section{Visual examples of real-life datasets}\label{sec:visual}

We visualize datasets that have meaningful features and whose trees are small enough to be contained in the paper.
We also adjust the minimum leaf size \revise{(1\% of the data size)} to produce smaller trees.

\subsection{Visualization of decision trees for shoppers dataset}
\begin{figure}[H]
    \centering
    \subcaptionbox{\cfourfive (AUC ROC 0.841 and expected height 8.9)}{
    	\picinput[width=.7\textwidth, height=.85\textheight, keepaspectratio]{viz_shoppers_C45}
    }
    \caption{\label{fig:viz:speed-dating} Visualization of \cfourfive decision tree for shoppers dataset.
    }
\end{figure}
\begin{figure}[H]
	\ContinuedFloat 
    \centering
    \subcaptionbox{\ecfourfive (AUC ROC 0.862 and expected height 5.49)}{
    	\picinput[width =.99\textwidth]{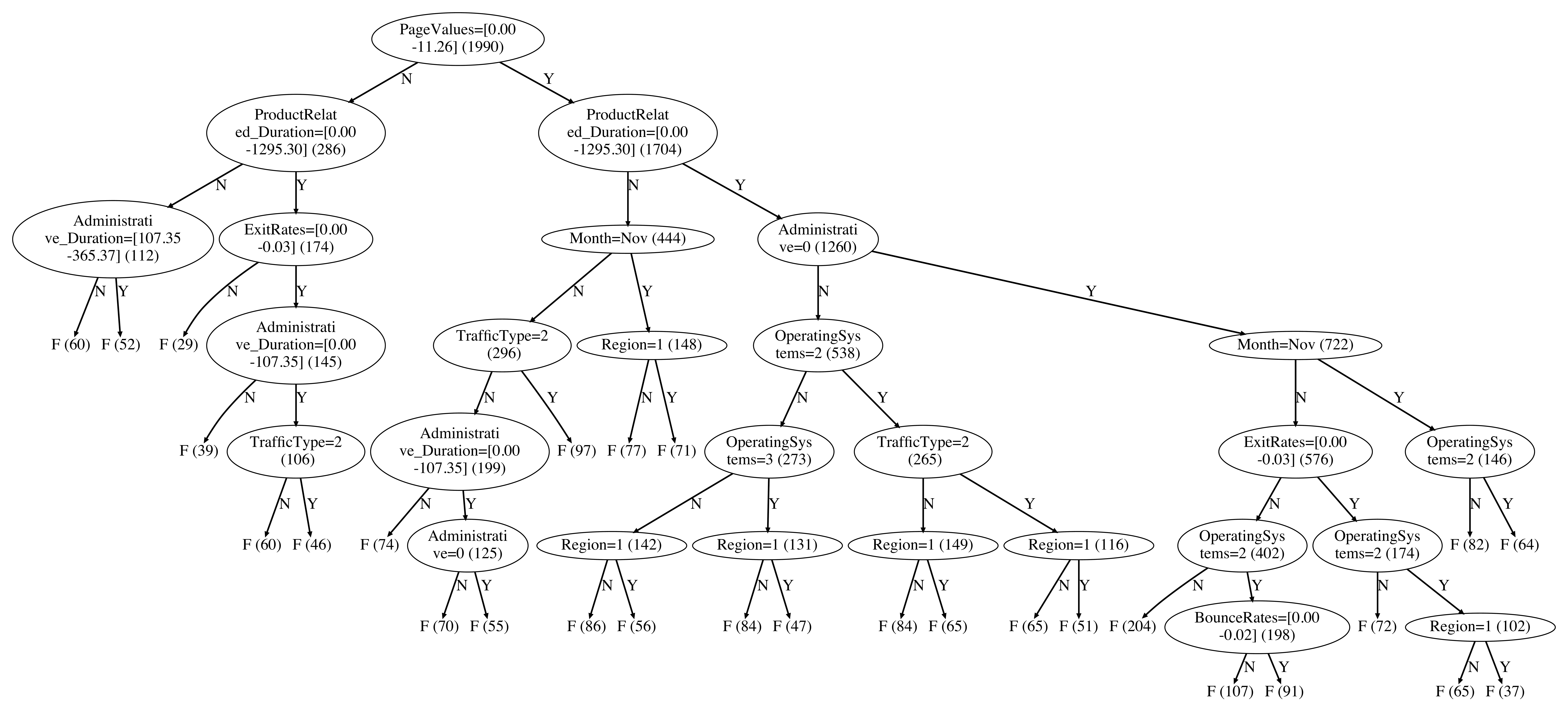}
    }
    \caption{\label{fig:viz:speed-dating} Visualization of \ecfourfive decision tree for shoppers dataset.
    }
\end{figure}
\begin{figure}[H]
	\ContinuedFloat 
    \centering
    \subcaptionbox{\asr (AUC ROC 0.666 and expected height 5)}{
    	\picinput[width =.99\textwidth]{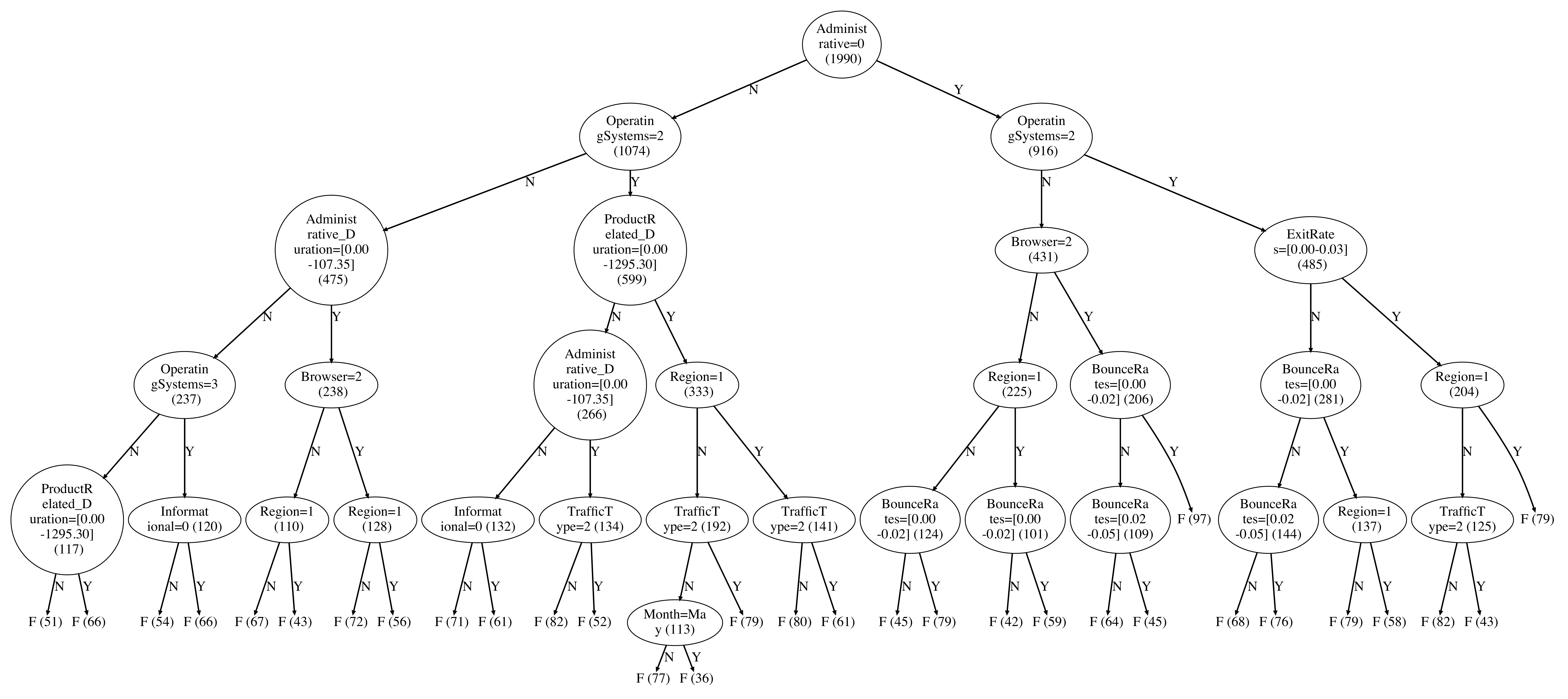}
    }
    \caption{\label{fig:viz:speed-dating} Visualization of \asr decision tree for shoppers dataset.
    }
\end{figure}

\subsection{Visualization of decision trees for breast-w dataset}
\begin{figure}[H]
    \centering
    \subcaptionbox{\cfourfive (AUC ROC 0.968 and expected height 3.5)}{
    	\picinput[width =.7\textwidth]{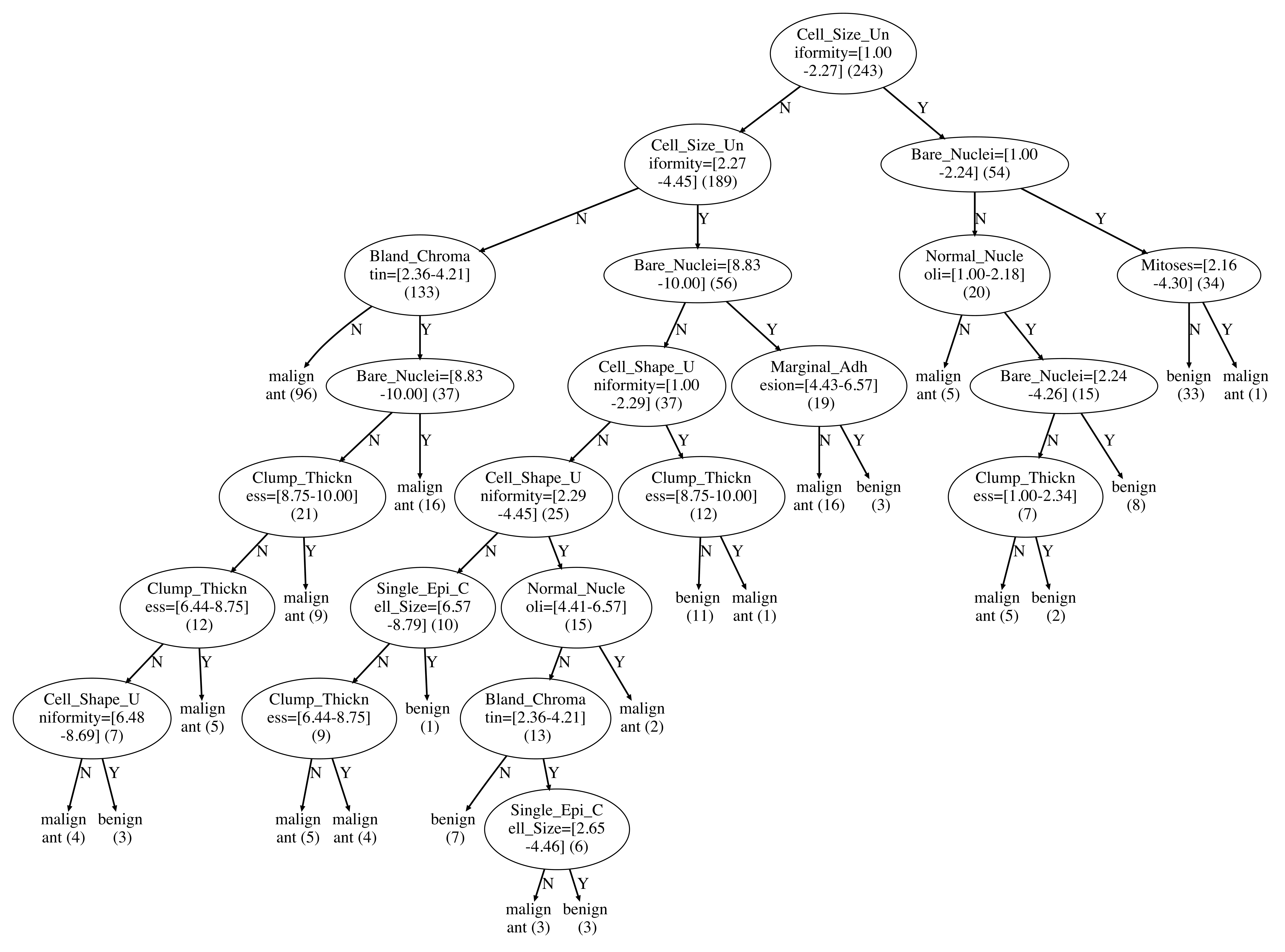}
    }
    \caption{\label{fig:viz:breast} Visualization of \cfourfive decision tree for breast-w dataset.
    }
\end{figure}
\begin{figure}[H]
	\ContinuedFloat 
    \centering
    \subcaptionbox{\ecfourfive (AUC ROC 0.982 and expected height 3.48)}{
    	\picinput[width =.7\textwidth]{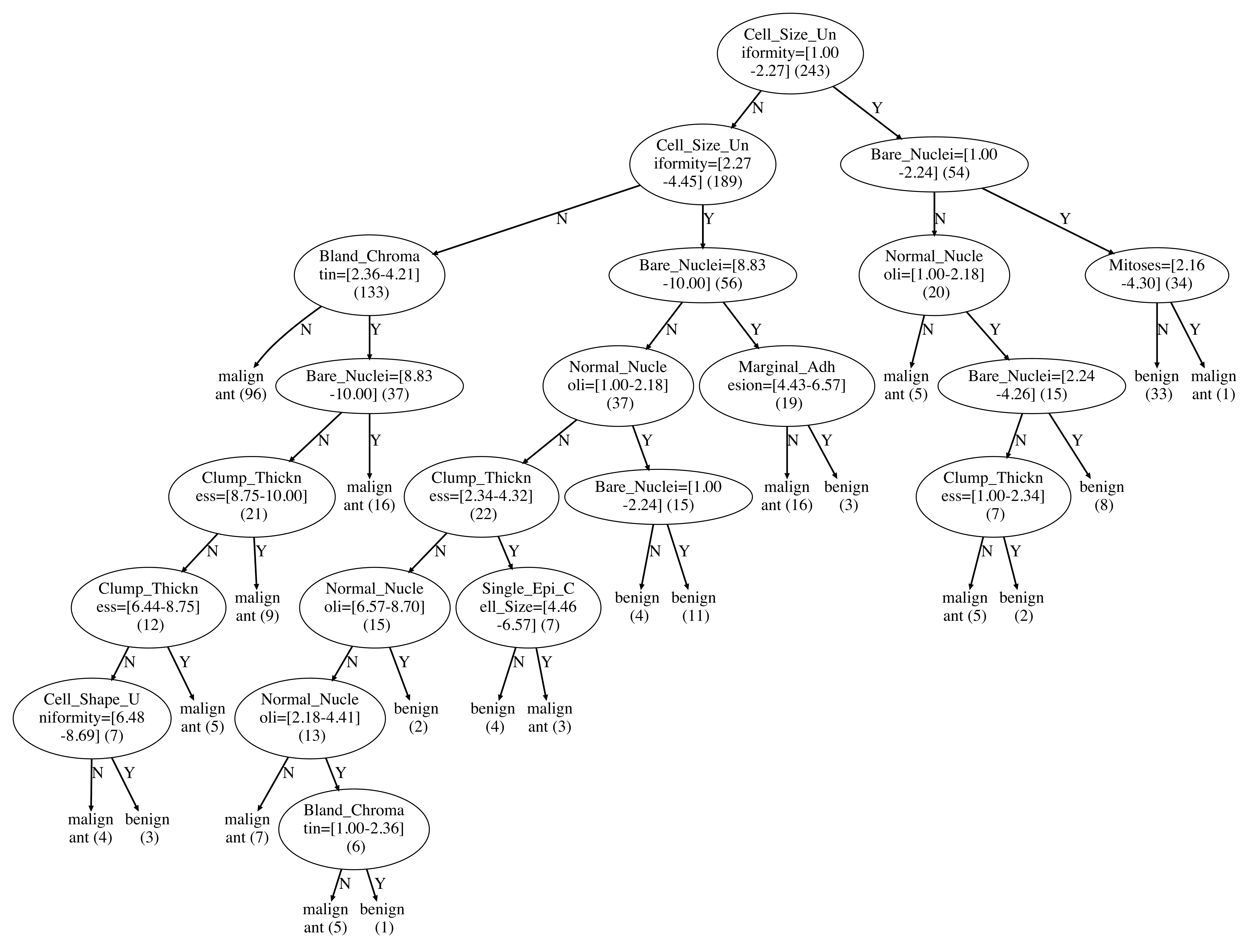}
    }
    \caption{\label{fig:viz:breast} Visualization of \ecfourfive decision tree for breast-w dataset.
    }
\end{figure}
\begin{figure}[H]
	\ContinuedFloat 
    \centering
    \subcaptionbox{\asr (AUC ROC 0.967 and expected height 4.08)}{
    	\picinput[width =.99\textwidth]{viz_breast-w_asr}
    }
    \caption{\label{fig:viz:breast} Visualization of \asr decision tree for breast-w dataset.
    }
\end{figure}

\subsection{Visualization of decision trees for obesity dataset}
\begin{figure}[H]
    \centering
    \subcaptionbox{\cfourfive (AUC ROC 0.932 and expected height 5.7)}{
    	\picinput[width =.99\textwidth]{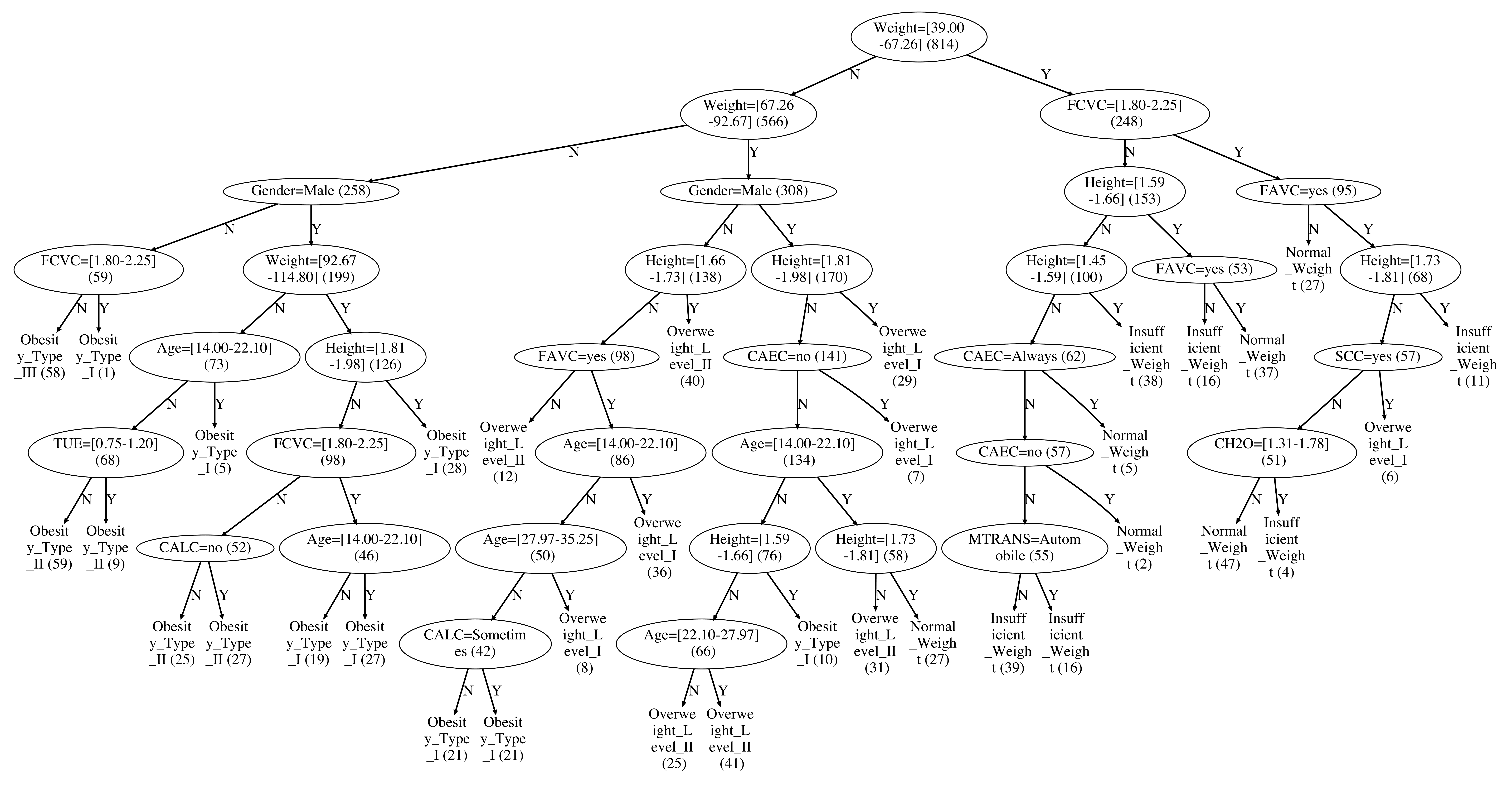}
    }
    \caption{\label{fig:viz:obesity} Visualization of \cfourfive decision tree for obesity dataset.
    }
\end{figure}
\begin{figure}[H]
	\ContinuedFloat 
    \centering
    \subcaptionbox{\ecfourfive (AUC ROC 0.930 and expected height 5.1)}{
    	\picinput[width =.99\textwidth]{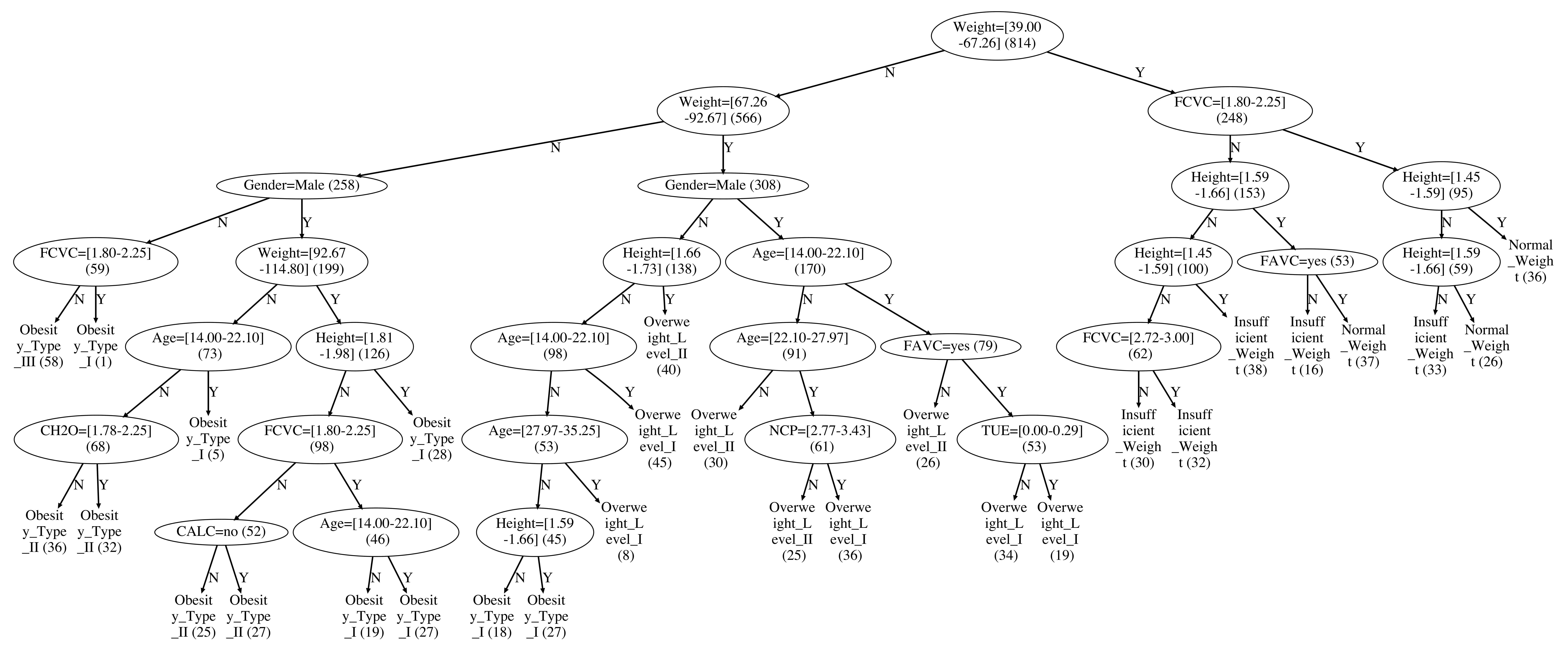}
    }
    \caption{\label{fig:viz:obesity} Visualization of \ecfourfive decision tree for obesity dataset.
    }
\end{figure}
\begin{figure}[H]
	\ContinuedFloat 
    \centering
    \subcaptionbox{\asr (AUC ROC 0.822 and expected height 4.9)}{
    	\picinput[width =.99\textwidth]{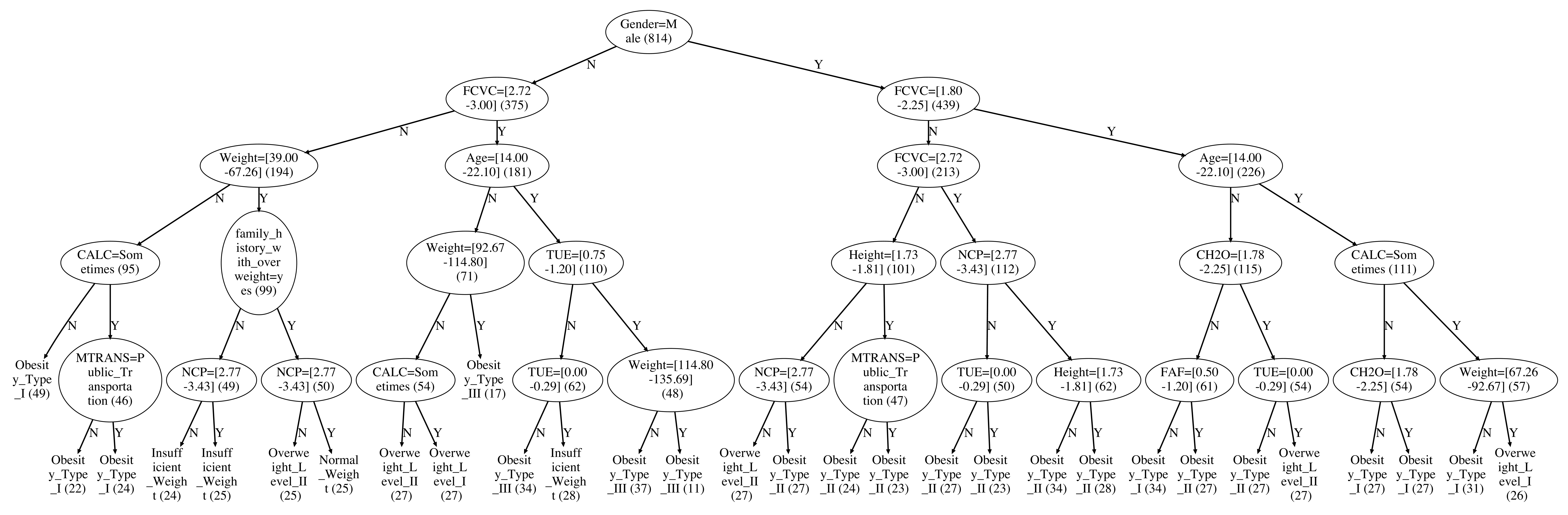}
    }
    \caption{\label{fig:viz:obesity} Visualization of \asr decision tree for obesity dataset.
    }
\end{figure}

\subsection{Visualization of decision trees for spambase dataset}
\begin{figure}[H]
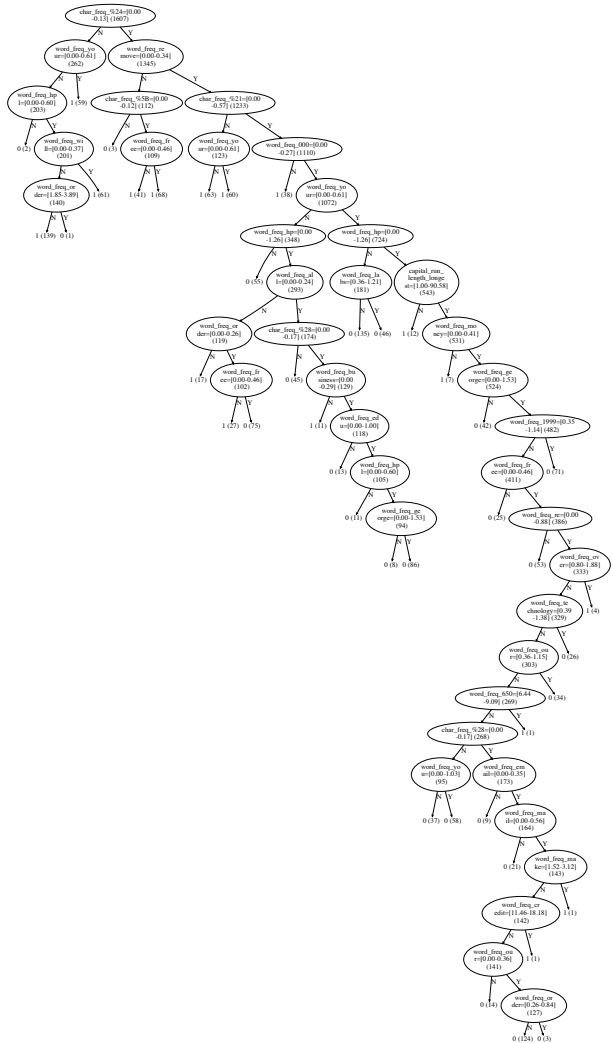

    \centering
    \subcaptionbox{\cfourfive (AUC ROC 0.932 and expected height 9.9)}{
    	\picinput[width = .68\textwidth, height=.85\textheight, keepaspectratio]{viz_spambase_C45}
    }
    \caption{\label{fig:viz:spambase} Visualization of \cfourfive decision tree for spambase dataset.
    }
\end{figure}
\begin{figure}[H]
	\ContinuedFloat 
    \centering
    \subcaptionbox{\ecfourfive (AUC ROC 0.928 and expected height 7.76)}{
    	\picinput[width =.99\textwidth]{viz_spambase_ours}
    }
    \caption{\label{fig:viz:spambase} Visualization of \ecfourfive decision tree spambase dataset.
    }
\end{figure}
\begin{figure}[H]
	\ContinuedFloat 
    \centering
    \subcaptionbox{\asr (AUC ROC 0.839 and expected height 5.8)}{
    	\picinput[width =.99\textwidth]{viz_spambase_asr}
    }
    \caption{\label{fig:viz:spambase} Visualization of \asr decision tree for spambase dataset.
    }
\end{figure}

\subsection{Visualization of decision trees for speed-dating dataset}
\begin{figure}[H]
    \centering
    \subcaptionbox{\cfourfive (AUC ROC 0.779 and expected height 5.9)}{
    	\picinput[width =.99\textwidth]{viz_speed-dating_C45}
    }
    \caption{\label{fig:viz:speed-dating} Visualization of \cfourfive decision tree for speed-dating dataset.
    }
\end{figure}
\begin{figure}[H]
	\ContinuedFloat 
    \centering
    \subcaptionbox{\ecfourfive (AUC ROC 0.790 and expected height 4.9)}{
    	\picinput[width =.99\textwidth]{viz_speed-dating_ours}
    }
    \caption{\label{fig:viz:speed-dating} Visualization of \ecfourfive decision tree for speed-dating dataset.
    }
\end{figure}
\begin{figure}[H]
	\ContinuedFloat 
    \centering
    \subcaptionbox{\asr (AUC ROC 0.569 and expected height 5)}{
    	\picinput[width =.99\textwidth]{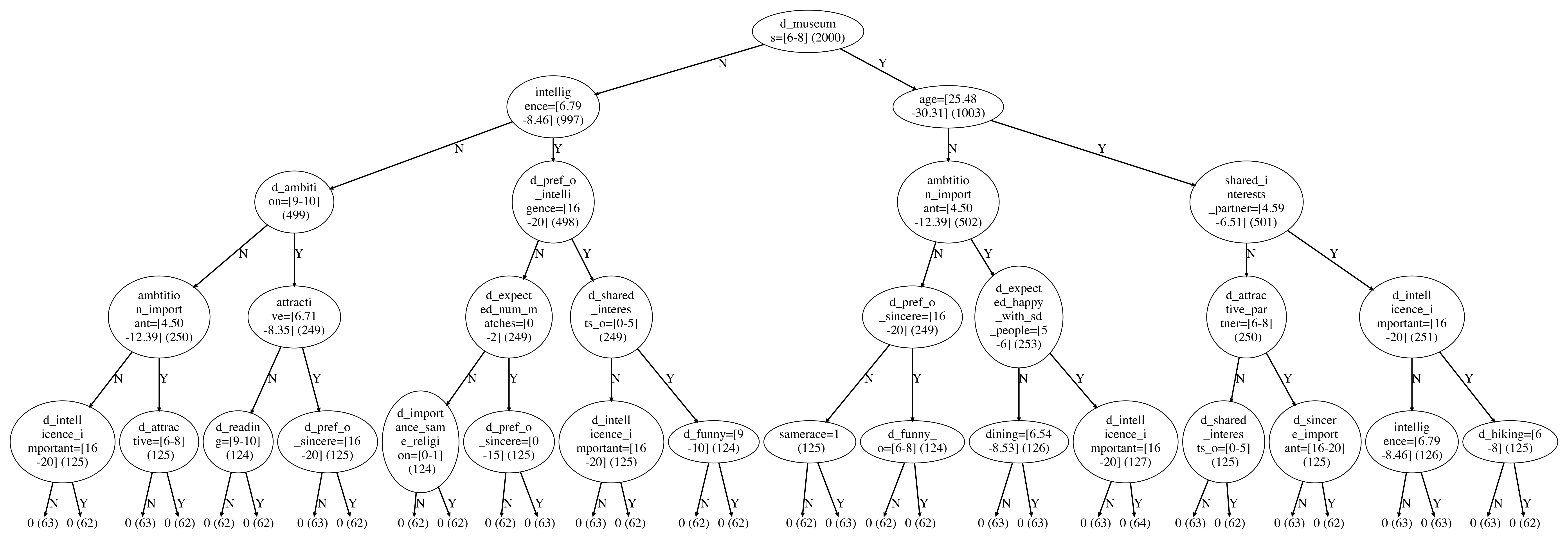}
    }
    \caption{\label{fig:viz:speed-dating} Visualization of \asr decision tree for speed-dating dataset.
    }
\end{figure}